\documentclass[nohyperref]{article}

\usepackage{microtype}
\usepackage{graphicx}
\usepackage{booktabs} 
\usepackage[dvipsnames]{xcolor}
\usepackage{caption}
\usepackage{subcaption}
\usepackage{placeins}
\usepackage{stfloats}
\usepackage{soul}

\definecolor{mydarkblue}{rgb}{0,0.08,0.45}
\usepackage[colorlinks=true,
    linkcolor=mydarkblue,
    citecolor=mydarkblue,
    filecolor=mydarkblue,
    urlcolor=mydarkblue]{hyperref}  

\usepackage[accepted]{icml2023}

\usepackage{amsmath}
\usepackage{amssymb}
\usepackage{mathtools}
\usepackage{amsthm}

\usepackage{dcolumn}
\usepackage{xcolor}
\usepackage{enumitem}

\newcolumntype{d}[1]{D{.}{.}{#1}}
\makeatletter
\newcolumntype{B}[3]{>{\boldmath\DC@{#1}{#2}{#3}}c<{\DC@end}}
\makeatother
\newcommand\mc[1]{\multicolumn{1}{c}{#1}}
\newcommand\boldc[1]{\multicolumn{1}{B{.}{.}{2.4}}{#1}}
\newcommand{\spm}[1]{{\scriptscriptstyle\pm#1}}
\newcommand{\spmi}[3]{\mathit{#1}.\mathit{#2}{\scriptscriptstyle\pm\mathit{#3}}}

\usepackage{amsmath,amsfonts,bm,amssymb,amsthm}

\def\eqref#1{equation~\ref{#1}}

\def\1{\bm{1}}

\def\eps{{\epsilon}}

\def\va{{\bm{a}}}
\def\vb{{\bm{b}}}
\def\vc{{\bm{c}}}
\def\vd{{\bm{d}}}

\def\vu{{\bm{u}}}
\def\vv{{\bm{v}}}

\def\vx{{\bm{x}}}
\def\vy{{\bm{y}}}

\def\mA{{\bm{A}}}

\def\mC{{\bm{C}}}

\def\mI{{\bm{I}}}

\def\mK{{\bm{K}}}

\def\mP{{\bm{P}}}
\def\mQ{{\bm{Q}}}

\def\mS{{\bm{S}}}
\def\mT{{\bm{T}}}

\def\mV{{\bm{V}}}
\def\mW{{\bm{W}}}
\def\mX{{\bm{X}}}

\def\mZ{{\bm{Z}}}

\DeclareMathAlphabet{\mathsfit}{\encodingdefault}{\sfdefault}{m}{sl}
\SetMathAlphabet{\mathsfit}{bold}{\encodingdefault}{\sfdefault}{bx}{n}

\newcommand{\R}{\mathbb{R}}

\newcommand{\softmax}{\mathrm{softmax}}

\DeclareMathOperator*{\argmin}{arg\,min}

\def\R{{\mathbb{R}}}
\def\T{{\top}}

\def\sink{{\mathrm{sinkhorn}}}

\usepackage[capitalize,noabbrev]{cleveref}

\theoremstyle{plain}
\newtheorem{theorem}{Theorem}[section]
\newtheorem{proposition}[theorem]{Proposition}

\theoremstyle{definition}

\theoremstyle{remark}

\usepackage[textsize=tiny]{todonotes}

\icmltitlerunning{Unlocking Slot Attention by Changing Optimal Transport Costs}

\begin{document}
\twocolumn[
\icmltitle{Unlocking Slot Attention by Changing Optimal Transport Costs}



\icmlsetsymbol{equal}{*}

\begin{icmlauthorlist}
\icmlauthor{Yan Zhang}{equal,SAIT}
\icmlauthor{David W. Zhang}{equal,UvA}
\icmlauthor{Simon Lacoste-Julien}{SAIT,Mila,CIFAR}
\icmlauthor{Gertjan J. Burghouts}{TNO}
\icmlauthor{Cees G. M. Snoek}{UvA}
\end{icmlauthorlist}

\icmlaffiliation{SAIT}{Samsung - SAIT AI Lab, Montreal}
\icmlaffiliation{UvA}{University of Amsterdam}
\icmlaffiliation{Mila}{Mila, Université de Montreal}
\icmlaffiliation{CIFAR}{Canada CIFAR AI Chair}
\icmlaffiliation{TNO}{TNO}

\icmlcorrespondingauthor{Yan Zhang}{yan@cyan.zone}
\icmlcorrespondingauthor{David W. Zhang}{w.d.zhang@uva.nl}

\icmlkeywords{Machine Learning, ICML}

\vskip 0.3in
]



\printAffiliationsAndNotice{\icmlEqualContribution} 

\begin{abstract}


    Slot attention is a powerful method for object-centric modeling in images and videos.
    However, its set-equivariance limits its ability to handle videos with a dynamic number of objects because it cannot break ties.
    To overcome this limitation, we first establish a connection between slot attention and optimal transport.
    Based on this new perspective we propose \textbf{MESH} (Minimize Entropy of Sinkhorn): a cross-attention module that combines the tiebreaking properties of unregularized optimal transport with the speed of regularized optimal transport. 
    We evaluate slot attention using MESH on multiple object-centric learning benchmarks and find significant improvements over slot attention in every setting.
\end{abstract}
 
\section{Introduction}  \label{sec:introduction}

Suppose we have an image containing two cats and one dog. Given a query like [cat, cat, dog], our task is to provide instance-specific information for each query element, such as their positions in the image.
When constructing a neural network to solve this problem, \emph{cross-attention} is a natural choice to relate the queries to our image \citep{vaswani2017attention,wei2020cross}.
With such a model, the dog can be located perfectly, but our two queries for the cats inevitably end up with an undesirable result: the \emph{average} of the two cats' positions.
The problem is that with our model, multiple copies of the same query element \emph{must} have the same result \cite{zhang2022multisetequivariant}; it is impossible to receive different answers for the same query, even if the context makes it obvious what is desired.


Models that rely on cross-attention, such as slot attention \cite{locatello2020object}, can run into this issue when trying to extract objects from images and other data modalities. 
This has especially been a problem in the video domain due to occlusions and new objects appearing \citep{kipf2022conditional}.
For example, \citet{wu2022slotformer} observe in experiments that when there are two objects but five ``slots'' (equivalent to queries in the previous example), the three slots that do not have a specific object to model become nearly identical.
Once that has happened, the ``cat problem'' from before applies: these very similar slots (queries) must receive  essentially the same information.
Consequently, if multiple new objects appear in the video, these three slots can---at best---all bind to only a single object, which is clearly undesirable.


This issue is present because cross-attention is set-equivariant, a property traditionally considered desirable. However, \citet{zhang2022multisetequivariant} show that set-equivariance is too restrictive when applied to \emph{multisets}: sets with repeated elements allowed, which are prevalent in deep learning (background in \autoref{sec:background}).
This manifests itself in two related problems:
\begin{enumerate}

    \item \textbf{Soft assignments.} The model tends to mix several inputs into each slot (query) rather than making a hard decision of one slot corresponding to exactly one input. This leads to difficulties when the information from each input must be kept distinct.
    \item \textbf{Lack of tiebreaking.} Similar slots are processed similarly, so they will likely contain similar information. Multiple similar slots prefer to capture an average of the relevant inputs rather than each slot capturing a different input, which leads to the problem described earlier where two cats cannot be localized individually.

\end{enumerate}

To avoid these issues, a property called \emph{exclusive multiset-equivariance} is necessary \cite{zhang2022multisetequivariant}.
So far, only models from the Deep Set Prediction Networks family \cite{zhang2019dspn, zhang2022multisetequivariant} are known to have this property.
However, they lack the object-centric inductive bias useful for object-centric learning tasks where the set-equivariant slot attention \citep{locatello2020object} shines.
Fortunately, introducing even a single exclusively multiset-equivariant module in a model is enough to give the entire model this property.
In this paper, we develop a module that enhances cross-attention in order to make the object-centric slot attention exclusively multiset-equivariant, thereby addressing the problems of soft assignments and tiebreaking.

In particular, we will look towards the field of optimal transport for inspiration, which is a natural fit for attention models \cite{sander2022sinkformers}.
This is because optimal transport concerns itself with computing the ``best'' assignment from one set to another \emph{given some pairwise costs}~\citep{villani2009optimal}, while cross-attention \emph{learns} exactly such costs.
This optimal transport perspective is especially useful because some optimal transport algorithms are able to break ties, which makes them relevant for multiset-equivariance. Unfortunately, these algorithms tend to be slow and difficult to parallelize. On the other hand, solutions to the entropy-regularized optimal transport problem~\citep{cuturi2013sinkhorn} are fast but are unable to break ties. 
We aim to combine the best of both worlds.

\paragraph{Contributions.}


\begin{enumerate}
    \item We establish that slot attention \emph{already} uses an approximation of regularized optimal transport  (\autoref{sec:model}).
    We use this to motivate variants of slot attention where either the approximation, or both the approximation and regularization are removed---with the latter being exclusively multiset-equivariant. However, this comes with a significant speed penalty and a lack of gradients, which can hinder learning.
    
    \item To avoid these issues, we introduce the \textbf{MESH} idea: \underline{m}inimize the \underline{e}ntropy of \underline{S}ink\underline{h}orn (\autoref{sec:MESH}).
    It combines the benefits of our two proposed slot attention variants: speed, gradients, and exclusive multiset-equivariance.
    We also show why this is more effective for reducing entropy and maintaining useful gradients than Sinkhorn alone.
    
    \item We evaluate our method in slot attention (SA-MESH\footnote{Samesh is an Indian name meaning ``lord of equality'', fitting for our method addressing a problem with equal elements.}) on two object detection and two unsupervised object discovery tasks (\autoref{sec:experiments}).
    We find that our optimal transport-based variants generally outperform slot attention.
    Crucially, SA-MESH almost always has the best results---often by a significant margin.
    
\end{enumerate}

\section{Background} \label{sec:background}
Multisets are generalizations of sets by allowing repetitions of elements.
In deep learning, sets and multisets are represented as $\R^{n \times c}$ matrices with $n$ being the number of elements and $c$ the feature dimension per element.
The uniqueness property of sets is rarely enforced in deep learning, so most models should be thought of as operating on multisets rather than sets \cite{zhang2022multisetequivariant}.
These models must then be careful to not rely on the arbitrary order of the $n$ elements.
To \emph{guarantee} this, they should satisfy certain equivariances.

\subsection{Permutation equivariances}
The standard definition of permutation-equivariant (\textbf{set-equivariant}) functions $f$ states that a permutation of the input $\mX$ should result in the same permutation of the output \citep{zaheer2017deep}.
With $\Pi$ as the space of $n \times n$ permutation matrices:
\begin{equation}
\begin{split}
     \forall \mX \in \R^{n \times c}, \forall \mP \in \Pi: \\
     f(\mP \mX) = \mP f(\mX).
\end{split}
\end{equation}
However, this means that a set-equivariant function must always produce the same result when there are equal inputs \cite{zhang2022multisetequivariant}: $f([\va, \va]) = [\vc, \vd]$ is not possible for $\vc \neq \vd$.
\citet{zhang2022multisetequivariant} therefore introduce a more appropriate equivariance for multisets, \textbf{multiset-equivariance}:
\begin{equation}
\begin{split}
     \forall \mX \in \R^{n \times c}, \forall \mP_1 \in \Pi, {\color{Emerald}\exists \mP_2 \in \Pi}: \\
     f(\mP_1 \mX) = {\color{Emerald}\mP_2} f(\mX) {\color{Emerald}~\land~\mP_1 \mX = \mP_2 \mX}.
\end{split}
\end{equation}
It states that when there are interchangeable elements in $\mX$ (so $\mP_1 \mX = \mP_2 \mX$ for $\mP_1 \neq \mP_2$), then there are \emph{multiple} permutations of the output that are valid for achieving equivariance.
This relaxation of set-equivariance makes the tiebreaking in $f([\va, \va]) = [\vc, \vd]$ possible.

While this property specifies what happens with equal elements, the continuity of most machine learning models suggests that the primary benefit in practice is with \emph{similar} elements \citep{zhang2022multisetequivariant}: similar elements no longer have to result in similar outputs.
All set-equivariant models are also multiset-equivariant, which means that only models that are \emph{exclusively multiset-equivariant} (multiset-equivariant, but not set-equivariant) are capable of tiebreaking.
Unfortunately, most operations in the multiset learning literature are set-equivariant and thus unable to break ties.

\subsection{Slot attention}
Slot attention (SA)~\citep{locatello2020object} can be used to abstract the contents of an image into a multiset of ``slots''.
Each slot can be thought of as a ``container'' that combines related information from the input into a vector.
The model learns how to route information from the input into these slots---in object-centric learning, these slots often learn to represent individual objects.
It does this by alternating two steps: 1.\ cross-attention between the multisets of input features and slot features to compute updates for each slot, and 2.\ utilizing a GRU~\citep{cho2014properties} to apply the computed updates to the corresponding slots.

The input features are represented by a matrix $\mX \in \R^{n \times c}$, and the slots are randomly initialized as a matrix $\mZ^{(0)} \in \R^{m \times d}$, with $m$ being the number of slots and $d$ the number of dimensions per slot.
Cross-attention in slot attention utilizes the standard key-query-value mechanism to compute updates for each slot. The computed updates are then applied to the corresponding slots using a GRU update. The whole procedure is repeated for a fixed number of times, which is referred to as the number of slot attention iterations.
\begin{align}
    \mQ^{(l)} &= \mZ^{(l)} \mW_Q, \quad
    \mK = \mX \mW_K, \quad
    \mV = \mX \mW_V \\
    \mA^{(l)} &= \texttt{normalize}(\softmax(\mQ^{(l)} {\mK}^\T))
    \label{eq:sa_a}\\
    \mZ^{(l+1)} &= \mathrm{GRU}(\mZ^{(l)}, {\mA^{(l)}} \mV).
\end{align}

In \citet{locatello2020object}, $\softmax$ forces each sum over the $m$ slots to be 1 while \texttt{normalize} forces each sum over the $n$ input elements to be 1.
All operations used in slot attention are set-equivariant with respect to the slots, which makes the model set-equivariant \cite{locatello2020object}.
In this paper, we introduce a module to make slot attention exclusively multiset-equivariant, which allows it to break ties between similar slots.

Note that similar slots can arise due to a multitude of reasons, such as the ones we pointed out in the introduction; it is not necessary for the objects in the input to be similar. This is also why the random initialization of slots in slot attention is not sufficient for tiebreaking, because the slots can become similar \emph{after} the slot updates, after which they can no longer be separated easily.

\section{Connecting slot attention to optimal transport} \label{sec:model}

Optimal transport~\citep{villani2009optimal} identifies the most cost-effective method for redistributing mass between two distributions. This process typically involves sampling from both distributions, calculating the pairwise distances between the samples, and applying an optimal transport algorithm to determine the transport map that minimizes the total cost. In the context of cross-attention, we are comparing two multisets containing the input and slot features rather than two distributions, but the same principles of computing pairwise distances and finding an optimal transport map still apply. 
By making this connection between slot attention and optimal transport more explicit, we can leverage the powerful algorithms of optimal transport to enhance the performance of slot attention.

\vspace{-1mm}
\paragraph{Entropy-regularized optimal transport.}
A critical component when applying cross-attention in slot attention is the normalization of the attention matrix.
\autoref{eq:sa_a} first exponentiates all the entries, then normalizes one dimension of the attention matrix to sum to 1, then the other dimension to sum to 1.
This sequence of operations is also known as applying the Sinkhorn algorithm for a single step.
The Sinkhorn algorithm solves the entropy-regularized optimal transport problem \cite{cuturi2013sinkhorn} by repeatedly alternating these two normalizations, which results in a doubly stochastic matrix at convergence.
We can therefore think of \autoref{eq:sa_a} as approximating this entropy-regularized optimal transport (by using only one Sinkhorn iteration) to determine how the information from the input should be associated with the slots.
This connection between transformers (which use a slightly different normalization) and optimal transport has also been made by \citet{sander2022sinkformers}.

A simple extension is thus to consider slot attention using more than one Sinkhorn iteration, which we will refer to as \textbf{SA-SH}.
In this variant, instead of using a similarity score, a distance function $d$ (e.g.\ L2 norm\footnote{\citet{sander2022sinkformers} point out that for Sinkhorn, this is equivalent to the usual (negative) dot product used in cross-attention.}) is used to compute the attention matrix as follows:
\begin{align}
    C_{ij} &= d(\mQ_{i}, \mK_{j}) \\
    \mA &= \sink(\mC).
\end{align}
Although the proposed extension properly computes the optimal transport map, it remains limited by its set-equivariance and is unable to perform tiebreaking.
One additional detail is that the Sinkhorn algorithm must handle non-square matrices since the number of slots is usually much smaller than the number of inputs;
na\"ively applying the algorithm on such matrices does not converge.
We describe the details of how to handle this in \autoref{app:marginals} as they are not important to the following discussion.

\paragraph{Unregularized optimal transport.}
By using the full Sinkhorn algorithm, SA-SH replaces the usual attention matrix with the transport map of the regularized optimal transport problem. 
This raises the question whether other optimal transport algorithms can be used in the context of slot attention too.
A benefit of optimal transport without regularization is that it can be \emph{exclusively multiset-equivariant}: many algorithms naturally include tiebreaking, which leads to low entropy solutions.
By replacing entropy-regularized optimal transport with unregularized optimal transport, we can  make slot attention exclusively multiset-equivariant to avoid the issues we pointed out in \autoref{sec:introduction}.

We thus propose the \textbf{SA-EMD} (Earth Mover's Distance) variant, wherein we use the EMD algorithm \citep{bonneel2011displacement} that is part of the POT package \cite{flamary2021pot}.
The EMD algorithm provides a sparse solution to the unregularized optimal transport problem and has the ability to break ties. 
However, using the EMD algorithm also presents some challenges. The gradients of unregularized optimal transport problems are piecewise constant, which prevents learning of the cost matrix $\mC$.
Thus, we need to estimate gradients, for which many different techniques exist
\citep{gould2016differentiating,fung2021fixed,bai2019deep,pogancic2020blackbox}.
In practice, we observe that the gradient of the Sinkhorn operator is also a good descent direction for EMD.
We find that we obtain the best empirical results through:
\begin{equation}
    \mA = \mathrm{emd}(\mC) + \sink(\mC). \label{eq:emd}
\end{equation}
Another issue is that the EMD algorithm is relatively slow.
The standard solvers use a network simplex algorithm (a variant of the simplex algorithm for graphs), which is difficult to parallelize efficiently on GPUs.
Can we get the benefits of unregularized optimal transport with its exclusive multiset-equivariance, while still being fast and differentiable like entropy-regularized optimal transport?

\section{MESH: minimizing the entropy of Sinkhorn} \label{sec:MESH}

Previously, we focused on different ways of turning costs $\mC$ into a transport map $\mA$.
Now, we turn our attention to \emph{changing the costs themselves}, followed by using the computational efficiency of the Sinkhorn algorithm to compute the transport map for these modified costs.
A key difference between unregularized optimal transport and entropy-regularized optimal transport is the entropy in the resulting transport map.
\emph{The idea is to change the costs in such a way that even after entropy-regularized optimal transport, the entropy remains low.}
This objective allows the tiebreaking necessary for exclusive multiset-equivariance.
To implement this idea, we aim to find a new cost $\mC'$ that minimizes the entropy $H(P) = - \sum_{i,j} P_{ij} \log P_{ij}$.

\begin{align}
    \mathrm{MESH}(\mC) &= \argmin_{\mC'\in\mathcal{V}(C)} {H(\sink(\mC'))} \label{eqn:me} \\
    \mA &= \sink(\mathrm{MESH}(\mC)).
\end{align}
$\mathcal{V}(\mC)$ denotes a neighborhood around $\mC$, which is implicitly defined by how we implement the $\argmin$.
We refer to this variant of slot attention as \textbf{SA-MESH} (Minimize Entropy of Sinkhorn).
\autoref{eqn:me} changes the cost so that the resulting transport map has low entropy (preferring 0s and 1s) despite the entropy regularization.
The final transport map is calculated from this new cost matrix efficiently using the Sinkhorn algorithm.
In summary, the costs are changed to make the transport map look more like the output of unregularized optimal transport.

To solve this optimization problem, we propose to use gradient descent for a small, fixed number of steps starting from the original cost matrix $\mC$. The last point ensures that the new cost matrix remains close to $\mC$ so that the new optimal transport problem will be similar to the original one.
Differentiating through this can be done with standard automatic differentiation.
While the process described so far does not feature any tiebreaking explicitly (and would in fact struggle to minimize entropy successfully when exactly equal rows or columns are present in $\mC$), we can simply add a small amount of noise at the start of the optimization: 
$\mC'^{(0)} = \mC + \bm\eps, \bm\eps_{ij} {\sim} \mathcal{N}(0, 10^{-6})$.
Another important detail is to normalize the gradient to have a fixed norm: the small amount of noise is amplified (to break ties) only when slots are similar.
This is done without having to resort to large learning rates, which would impact the stability of optimization.
We thus propose to repeat the following for $T$ steps with a learning rate $\lambda$:
\begin{align}
    \mC'^{(t+1)} = \mC'^{(t)} - \lambda \frac{\nabla_{\mC'^{(t)}} \mA'^{(t)}}{\left\lVert\nabla_{\mC'^{(t)}} \mA'^{(t)}\right\rVert} \label{eqn:grad} \\
    \text{with } \mA'^{(t)} = H(\sink(\mC'^{(t)})).
\end{align}
Note that exact minimization is not needed, nor necessarily desirable.
A critical aspect is that $\mC'^{(T)}$ is still related to the initial $\mC$.
For example, if $\mC$ is square, then $\mC' = 10^{23}\mI$ would always ``successfully'' minimize \autoref{eqn:me}, but also solve a problem unrelated to $\mC$ and provide no signal to learn from.
This is why we initialize $\mC'^{(0)}$ as $\mC$ with a small amount of noise.
Note that if this noise is too high (e.g.\ $\bm\eps_{ij} {\sim} \mathcal{N}(0, 1)$), the solution can again lose correspondence with the actual cost matrix $\mC$ that we are trying to compute the transport map for. 
We develop a more complicated version that explicitly enforces $\mC$ and $\mC'$ to be similar (\autoref{app:similarity}) but find that it is not necessary in practice as long as the noise in the initialization is reasonably low.

\subsection{Properties}
\paragraph{Equivariance.}
SA-MESH is now exclusively multiset-equivariant:
it is multiset-equivariant because all the individual operations are multiset-equivariant, but it is not set-equivariant because equal slots will no longer receive the same transport plans due to the tiebreaking from the noise and subsequent optimization (see proof in \autoref{app:proof}).
This gives our method more representational power than standard SA and SA-SH because it is no longer restricted by set-equivariance.

A useful side-effect is that random initialization of slots is no longer necessary. 
Since ties can be broken by this entropy minimization, it is no problem to initialize all slots to be the same vector.
In standard slot attention, the amount of noise in $\mZ^{(0)}$ needs to be just right: too low, and the set-equivariant model has difficulties breaking ties between these similar slots~\citep{wu2022slotformer}; too high, and the model can become unreliable from the noisiness \citep{kipf2022conditional}.
Furthermore, randomly initializing the slots does not prevent them from collapsing to the same values in a later iteration---such as when there are fewer objects to model than slots.
In contrast, a tiny amount of noise in $\mC'$ (as long as it is above machine precision after $\mathrm{sinkhorn}$) is sufficient for tiebreaking and avoiding collapse, and any other effect of the noise can be optimized away through the gradient descent.



\paragraph{Speed.}
SA-MESH offers a significant improvement in computation time when compared to SA-EMD. This is because it only requires the evaluation of the Sinkhorn algorithm for a small number of optimization steps (in our experiments, we found little improvement above four steps). While it may not be as fast as SA-SH, its exclusive multiset-equivariance can speed up learning by making the model more powerful. In many cases, the computation required for SA-MESH is outweighed by other components of the model such as the image processing part that creates the input multiset for SA.
Since each optimization step in MESH needs to solve a similar optimal transport problem, we can reuse the result from the previous iterations to improve the efficiency, as we describe in \autoref{app:sinkhorn}.

\begin{figure*}
    \centering
    \includegraphics[width=0.49\linewidth, trim={0 5mm 0 0}]{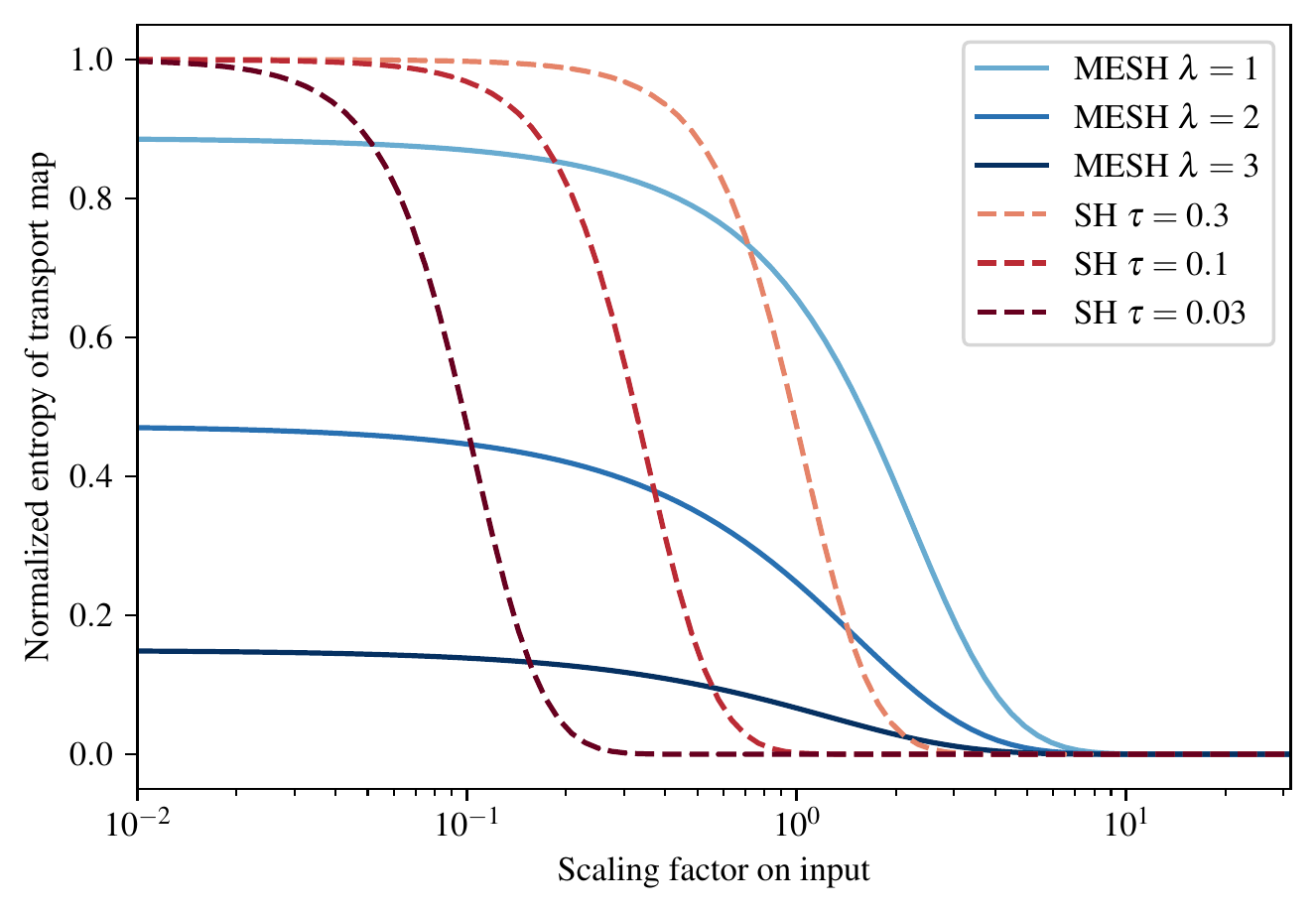}
    \includegraphics[width=0.49\linewidth, trim={0 5mm 0 0}]{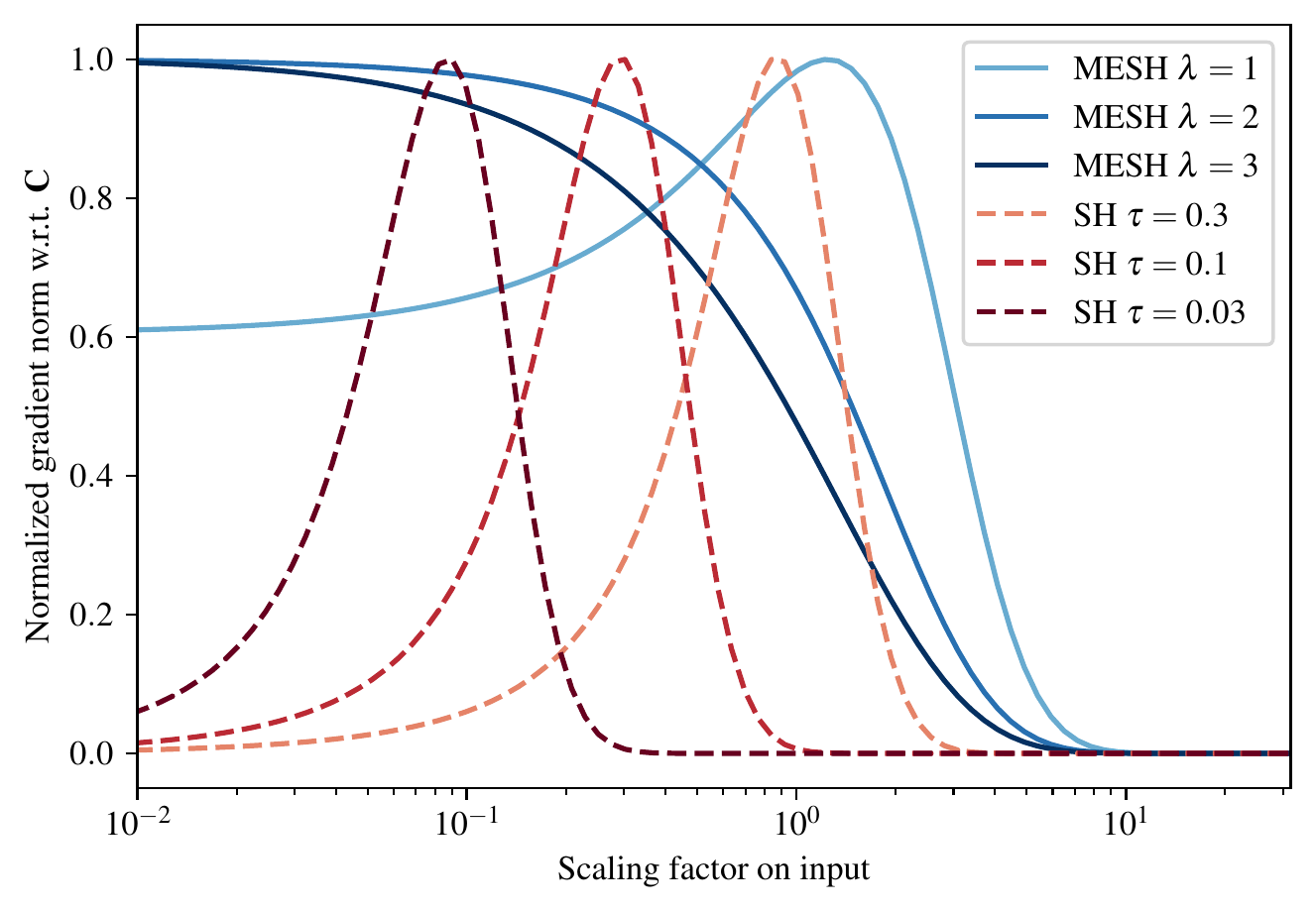}
    \caption{MESH reduces entropy while maintaining reasonable gradients across a large range of attention values, in particular when attention is uncertain (low scaling factors). Meanwhile, Sinkhorn (SH) provides nontrivial gradients in a much smaller range with entropy reduction being ineffective when the scaling is too low. Entropy of transport map (left) and corresponding gradient norm (right) when varying the scaling factor on the input. The entropy is normalized to have 1 as maximum entropy. The gradient norm of each method is normalized to have a maximum of 1, we show the unnormalized gradient norms in \autoref{app:tempsh vs mesh}.}
    \label{fig:grad}
\end{figure*}

\paragraph{Gradients.}
Another benefit over SA-EMD is that gradient computation is simple since we can fully rely on automatic differentiation instead of having to manually estimate gradients for the black-box EMD solver in SA-EMD.
We find experimentally that we do not even need to differentiate the gradient updates in \autoref{eqn:grad} themselves; the gradients of $\mC'^{(T)}$ can simply be passed along to $\mC$ in a straight-through manner \citep{bengio2013straightthrough} without reduction in performance. 
SA-MESH also has benefits in terms of the ``quality'' of gradients over SA-SH, which we explain in the following.

  

\subsection{Comparison to changing temperature}

An alternative for breaking ties is to add noise to the cost matrix (like SA-MESH), but then simply reduce the temperature of the Sinkhorn algorithm, which corresponds to reducing the amount of entropy regularization \citep{cuturi2013sinkhorn}.
For sufficiently low temperatures, this should also be able to map equal inputs to different slots.
How does this much simpler approach compare to MESH, which minimizes entropy by gradient descent?

Sinkhorn with low temperatures can be thought of as analogous to softmax with low temperatures. As the temperature decreases, the behavior of softmax becomes more similar to an argmax, but the gradients become ill-behaved as a result. Similarly, with Sinkhorn, low temperatures may result in gradients that make it difficult or impossible to learn a good cost matrix.
We now make this notion more concrete and show that MESH can reduce entropy \emph{while maintaining ``good'' gradients} for a much larger range of inputs.

In \autoref{fig:grad}, we have the following set-up.
Starting from a $10 \times 10$ identity matrix, we scale it by a varying factor to obtain the cost matrix, then apply either Sinkhorn (for different temperatures $\tau$) or MESH (for different learning rates $\lambda$).
We then measure the entropy of the resulting transport map (left), as well as the norm of the gradient of this entropy with respect to the cost matrix (right).
Gradient norms close to zero slow down learning because the gradients provide little to no information on how to learn the cost matrix.

When the scaling factor is large (e.g.\ $>\! 10$, attention is confident), the behaviors are similar (low entropy and small gradients).
When the scaling factor is small (e.g.\ $<\! 0.1$, attention is not confident), MESH is still always able to reduce entropy while maintaining reasonable gradients.
In contrast, Sinkhorn only has nontrivial gradients in a relatively small range; outside of this range, learning is difficult because the gradient norms become close to 0.
Thus, a trade-off has to be made for Sinkhorn, which is not necessary for MESH: either $\tau$ is high and low scaling factors result in virtually zero gradients (with no reduction to the entropy), or $\tau$ is low and higher scaling factors result in virtually zero gradients instead.
The MESH learning rate $\lambda$ can be used to control the shape of how entropy is reduced, while the Sinkhorn temperature $\tau$ only changes the location and does not increase the range of nontrivial gradients.
In summary, MESH effectively reduces entropy while maintaining well-behaved gradients for a large range of inputs.


\section{Experiments} \label{sec:experiments}
We now experimentally evaluate SA with our optimal transport variants, with a particular focus on comparing SA to SA-MESH.
We open-source all of our code {\url{https://github.com/davzha/MESH}} and provide extra experimental details in \autoref{app:experiment details}.

\begin{table}
    \centering
    \caption{Random objects detection, measured in RMSE divided by standard deviation $\sigma$ of random objects (lower is better). An always-predict-zeros baseline has a normalized RMSE of 1. Median over 5 random seeds.}
    \label{tab:copying}
\resizebox{0.8\linewidth}{!}{
    \begin{tabular}{lccc}
    \toprule
    {Model} & $\sigma=1.0$ & $\sigma=0.1$ & $\sigma=0.01$ \\
    \midrule
    SA      & 0.44 & 0.53 & 0.65 \\
    SA-SH   & 0.27 & 0.32 & 0.41 \\
    SA-EMD  & \textbf{0.23} & 0.52 & 1.04 \\
    SA-MESH & {0.24} & \textbf{0.27} & \textbf{0.31}  \\
    \bottomrule
    \end{tabular}
    }
\end{table}

\begin{table*}[b]
    \centering
    \caption{
        CLEVR property prediction,
        average precision (AP) in \% (mean $\pm$ standard deviation) over 5 random seeds, higher is better.
        SA-MESH improves over all other SA variants at only a small computational cost. Note that the exclusively multiset-equivariant iDSPN is not object-centric, so it is not fully comparable.
        SA (original) results are copied from \citet{locatello2020object}, iDSPN results from \citet{zhang2022multisetequivariant}.
        Models with $\dagger$ use the improvement by \citet{chang2022object}, see \autoref{app:clevr extra} for our results without $\dagger$.
    }
    \label{tab:state}
\resizebox{0.90\textwidth}{!}{
    \begin{tabular}{l *{6}{d{2.4}}c}
        \toprule
        \mc{Model} & \mc{AP\textsubscript{$\infty$}} & \mc{AP\textsubscript{1}} & \mc{AP\textsubscript{0.5}} & \mc{AP\textsubscript{0.25}} & \mc{AP\textsubscript{0.125}} & \mc{AP\textsubscript{0.0625}} & Train time \\

        \midrule
        \textit{iDSPN} \cite{zhang2022multisetequivariant} & \spmi{98}{8}{0.5} & \spmi{98}{5}{0.6} & \spmi{98}{2}{0.6} & \spmi{95}{8}{0.7} & \spmi{76}{9}{2.5} & \spmi{32}{3}{3.9} & \mc{---}\\
        SA (original) \cite{locatello2020object} & 94.3\spm{1.1} & 86.7\spm{1.4} & 56.0\spm{3.6} & 10.8\spm{1.7} & 0.9\spm{0.2} & \mc{---} & \mc{---}\\
        \midrule

        SA$\dagger$ & 94.3\spm{0.4} & 85.7\spm{1.6} & 77.2\spm{1.5} & 53.1\spm{2.7} & 16.7\spm{1.8} & 4.0\spm{0.7} & 2.2 h\\
        SA-SH$\dagger$ & 98.9\spm{0.2} & 97.7\spm{0.5} & 95.2\spm{0.9} & 83.3\spm{0.8} & 38.5\spm{2.0} & 10.0\spm{1.4} & 2.3 h \\
        SA-EMD$\dagger$ & 99.3\spm{0.3} & 98.1\spm{0.4} & 95.9\spm{0.8} & 85.8\spm{1.1} & 42.0\spm{2.0} & 11.4\spm{1.3} & 9.3 h \\
        \textbf{SA-MESH}$\dagger$ & \boldc{99.4\spm{0.1}} & \boldc{99.2\spm{0.2}} & \boldc{98.9\spm{0.2}} & \boldc{91.1\spm{1.1}} & \boldc{47.6\spm{0.8}} & \boldc{12.5\spm{0.4}} & 2.4 h \\
        \bottomrule
    \end{tabular}%
}

\end{table*}

\subsection{Random objects detection}\label{sec:copying}
First, we evaluate the effect of the different SA variants in a simplified object detection setting.
The aim is to assess a model's ability to detect and distinguish similar objects in a controlled setting.
Given a multiset containing $k$ random 32d vectors sampled from $\mathcal{N}(0, \sigma^2 \mI)$ and $h$ zero vectors, the goal is to copy only the $k$ random vectors into the $k$ slots.
In the context of object detection on images, this can be thought of as detecting $k$ different objects, each of which occupies exactly one position in a feature map.

In this setting, we want the object information to be preserved as accurately as possible.
To vary the difficulty of this task, we change the standard deviation $\sigma$ of the random vectors to be copied and measure the error relative to this $\sigma$.
A lower standard deviation corresponds to a harder task because the elements become more similar to each other and to the background zero vectors.

\paragraph{Results.}
\autoref{tab:copying} shows our results for $k=5$ objects and $h=100$ background elements for varying difficulties~$\sigma$.
First, we see that SA-SH and SA-MESH can detect the objects on the hardest setting of $\sigma=0.01$ more accurately than SA on $\sigma=1$.
This demonstrates the general value of our proposed optimal transport perspective in the context of attention.
SA-MESH with its lower entropy outperforms SA-SH: the SA-MESH performance for a specific $\sigma$ is roughly equivalent to SA-SH at a $\sigma$ ten times higher.
This shows the benefits of reducing entropies in the transport map through MESH, which helps objects stay distinct from each other.
Meanwhile, SA-EMD (also with low entropy transport maps) performs similarly to SA-MESH on $\sigma=1$, but degenerates to the always-predict-zeros baseline on $\sigma=0.01$.
We attribute this to the inherently imprecise gradient estimation leading to learning problems.






\subsection{CLEVR property prediction}\label{sec:predict}
Next, we test SA and our proposed variants on a more realistic object detection task.
CLEVR \cite{johnson2017clevr} is a synthetic dataset containing images with up to ten objects in a 3d scene.
Each object is sampled with varying sizes, materials, shapes, and colors.
The task is to predict the multiset of objects with their properties and 3d position.
Following \citet{zhang2019dspn}, we evaluate using average precision (AP) at different distance thresholds for the 3d coordinates of the predicted objects.

\paragraph{Results.}
\autoref{tab:state} shows that SA-MESH achieves the best SA results (and state-of-the-art results on some metrics) while only increasing run time by a small amount.
These results are followed by SA-EMD, which has slightly worse results (likely due to the inherently imprecise gradient estimation) but takes over three times longer to train.
Again, we see that all three optimal transport-based methods greatly outperform the baseline SA, which validates the benefits of the optimal transport perspective in attention.

We believe that there are two reasons why SA-SH performs much better than SA, even though both are set-equivariant. Keep in mind that SA is equivalent to SA-SH with a single Sinkhorn iteration. Performing more Sinkhorn iterations gives a more accurate transport map, which can be interpreted as fully resolving the ``competition'' between slots for inputs. This competition interpretation is what motivated the normalizations (1-step Sinkhorn) in \citep{locatello2020object}.
The other factor is that SA-SH, in order to converge for rectangular cost matrices, requires the use of learned marginals (see \autoref{app:marginals}). These can assist with weighting down the importance of background pixels or unused slots, which makes it easier for the model to learn the $\mW_K$ and $\mW_V$ matrices of slot attention.



Note that we only provide the results for iDSPN (which is also exclusively multiset-equivariant) for context; this is not supposed to be a direct comparison due to the significant difference in approach.
In general, slot attention through the use of attention has the benefit of not needing to compress the input into a single vector (global scene representation) like iDSPN.
The resulting object-centric inductive bias and the relative simplicity have allowed for wider adoption and success of SA over iDSPN \cite{kipf2022conditional,hu2020sas,li2021scouter,sajjadi2022osrt}, which makes our improvements to SA meaningful, even if in this case some metrics are worse than iDSPN.


\subsection{Unsupervised object discovery on images}\label{sec:od}
In this task, the objective is to discover objects without the supervision of what the objects are.
We follow \citet{locatello2020object} and set up an image reconstruction task with slots as the latent bottleneck using SA.
These slots are individually decoded into object-specific images, each comprising the RGB color channels and an alpha mask. These object-specific images are then combined to form the final reconstructed image. 
To evaluate the performance, we compare the per-slot alpha masks to the actual object segmentation masks. 
The goal is thus for image reconstruction with a multiset bottleneck to lead to a decomposition of the scene into individual objects, with each object being modeled by a distinct slot.

We evaluate on the Multi-dSprites dataset, which is the only benchmark presented by~\citet{locatello2020object} that still presented a challenge (possibly due to the presence of highly overlapping objects). Additionally, we test on ClevrTex~\citep{clevrtex}, a synthetic dataset similar to CLEVR that introduces the added challenge of different textures.
In line with prior work~\citep{kipf2022conditional}, we evaluate the Foreground Adjusted Rand Index (FG-ARI) and Foreground mean Intersection over Union (FG-mIoU). To ensure FG-mIoU is permutation-insensitive, we use the Hungarian algorithm to find the best matching between the masks. See~\autoref{app:object discovery details} for more details.
Note that we no longer test SA-EMD because the larger input size compared to~\autoref{sec:predict} makes its training time infeasible.


\begin{table}
    \centering
    \caption{Object discovery on images results in Multi-dSprites in \% (mean $\pm$ standard deviation) over 5 random seeds, higher is better. SA-MESH outperforms the other models and has lower variance.}
    \label{tab:mdsprites}
    
    \resizebox{0.9\linewidth}{!}{
    \begin{tabular}{l*{2}{d{2.4}}}
        \toprule
         Model & \mc{FG-ARI} & \mc{FG-mIoU} \\
         \midrule
         SA \citep{locatello2020object}&91.3\spm{0.3}&\mc{---}\\
         \midrule
         SA & 92.2\spm{0.5} & 24.3\spm{5.4} \\
         SA-SH & 87.2\spm{1.8} & 84.0\spm{3.1} \\
         SA-MESH & \boldc{95.6\spm{0.2}}& \boldc{86.2\spm{2.2}}\\
         \bottomrule
    \end{tabular}
    }
    \end{table}


\begin{table}
    \centering
    \caption{Object discovery on images results in ClevrTex in \% (mean $\pm$ standard deviation) over 5 random seeds, higher is better. SA-MESH outperforms the other models and has lower variance.}
    \label{tab:clevrtex}
    \resizebox{0.65\linewidth}{!}{
    \begin{tabular}{l*{2}{d{2.4}}}
        \toprule
         Model & \mc{FG-ARI} & \mc{FG-mIoU} \\
         \midrule
         SA & 52.8\spm{14.9} & 26.3\spm{14.9}\\
         SA-SH & 70.8\spm{5.8} & 35.3\spm{4.4} \\
         SA-MESH & \boldc{79.0\spm{2.6}} & \boldc{43.2\spm{4.0}}\\
         \bottomrule
    \end{tabular}
    }
\end{table}

\paragraph{Results.}
\autoref{tab:mdsprites} and \autoref{tab:clevrtex} show that on both Multi-dSprites and ClevrTex, SA-MESH achieves significantly higher FG-ARI and FG-mIoU compared to all baselines. SA-SH and SA-MESH improve especially in mIoU, which \citet{clevrtex} argue is a better metric than ARI to evaluate the accuracy of object masks.
\autoref{app:extra experiments} shows two extra ablations on the training setup.

\begin{figure}
    \centering
    \includegraphics[width=0.98\linewidth, trim={15mm 0mm 15mm 0mm}, clip]{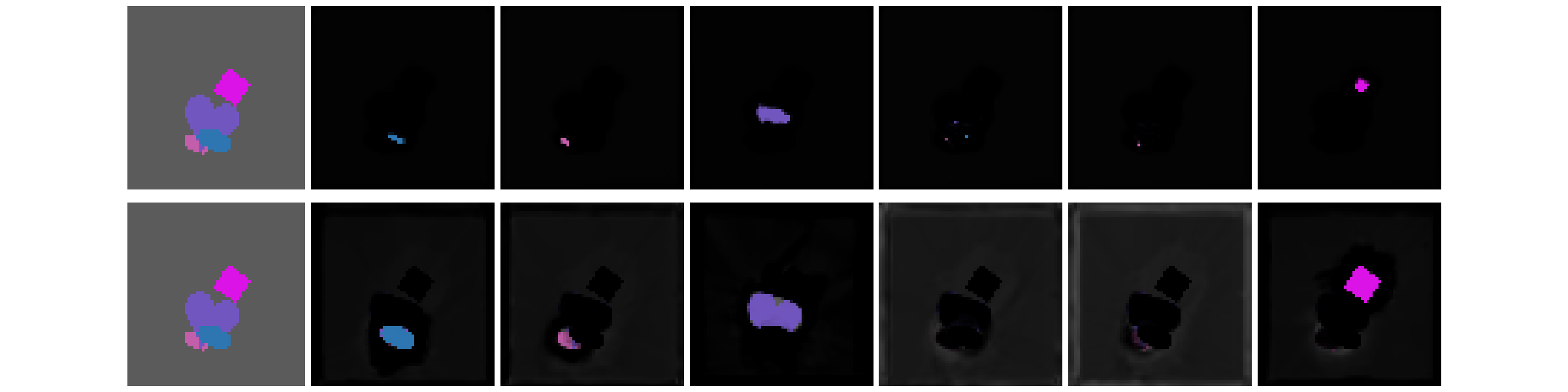}
    \caption{Attention maps (first row) and alpha masks (second row) for each of the six slots in SA-MESH on Multi-dSprites. The attention focuses on only a small region inside each object, but the model still reconstructs the full objects. See \autoref{app:od examples} for more examples.} \label{fig:attn maps and masks}
\end{figure}

\paragraph{Analysis.}
In~\autoref{fig:attn maps and masks}, we observe that minimizing the entropy leads to much sparser attention maps for SA-MESH in comparison to SA:
SA-MESH only attends to a small part within each object, rather than the whole object as is usually the case for SA.
Keep in mind that attention maps are only used to route the required information to each slot, and it is the slots themselves that represent the individual objects.
Attending to a small part of each object is sufficient because the receptive field of the CNN image encoder lets it move information from the edge of an object into the center.
The accuracy of the final alpha masks of SA-MESH shows that this happens successfully.
In \autoref{app:od examples}, we see that SA-MESH attends less to the background than SA on Multi-dSprites.


Similar to \citet{locatello2020object}, we observe that some SA runs fail to separate the individual objects into different slots despite having a low reconstruction error.
In those cases, the attention maps divide the image into separate regions, independent of the image content. 
SA-MESH ensures that the attention maps are sparse, which helps to avoid these kinds of failure modes.
\autoref{app:od examples} shows this difference between SA-MESH and SA on ClevrTex: SA-MESH discovers objects successfully in most cases, and usually, only the background is split into spatial regions.

\begin{table*}[t]
    \centering
    \caption{Video object discovery results on CLEVRER-S and CLEVRER-L in \% (mean $\pm$ standard deviation) over 5 random seeds. SA-MESH outperforms all other models in terms of quality of predicted masks (ARI, mIoU) and achieves high temporal consistency (TC). See \autoref{app:od examples} for example masks.}
    
    \resizebox{0.75\linewidth}{!}{
    \begin{tabular}{l*{6}{d{2.4}}}
        \toprule
        &\multicolumn{3}{c}{CLEVRER-S}&\multicolumn{3}{c}{CLEVRER-L}\\
        \cmidrule(lr){2-4} \cmidrule(lr){5-7}
         Model&\mc{FG-ARI}&\mc{FG-mIoU}&\mc{TC}&\mc{FG-ARI}&\mc{FG-mIoU}&\mc{TC} \\
         \midrule
         SA & 
         78.1\spm{14.0} & 16.8\spm{9.8} & 26.3\spm{22.5} & 
         69.6\spm{14.9} & 12.2\spm{6.6} & 12.8\spm{8.7} \\
         SA fixed noise & 
         71.0\spm{34.0} & 17.1\spm{11.4} & 42.8\spm{19.3} &
         79.4\spm{5.9} & 11.9\spm{6.2} & 18.5\spm{13.0} \\
         SA learned noise & 
         80.2\spm{14.1} & 13.2\spm{4.9} & 21.3\spm{13.5} &
         84.7\spm{5.7} & 10.2\spm{1.2} & 19.4\spm{7.7} \\
         SA-SH & 
         89.3\spm{2.3} & 10.0\spm{2.8} & \boldc{89.7\spm{2.4}} &
         82.9\spm{1.5} & 7.0\spm{0.2} & 26.8\spm{0.5} \\
         SA-MESH & 
         \boldc{93.8\spm{1.0}} & \boldc{44.1\spm{7.1}} & 80.2\spm{10.1} & 
         \boldc{92.9\spm{2.2}} & \boldc{54.4\spm{8.9}} & \boldc{55.4\spm{5.7}} \\
         \bottomrule
    \end{tabular}
    }
    \label{tab:CLEVRER results}
\end{table*}

\subsection{Unsupervised object discovery on video}\label{sec:od video}
As we mentioned in \autoref{sec:introduction}, \citet{wu2022slotformer} observed issues with SA when applied to videos where multiple objects can enter the scene.
To evaluate our method in this scenario, we build two variants of the CLEVRER video dataset~\citep{yi2019clevrer} where the number of visible objects varies over time.
We only use two frames from each video: the first frame, and either the 16th frame (short time difference, CLEVRER-S) or the 128th frame (long time difference, CLEVRER-L).
We do this to evaluate SA without the dynamics prediction component that \citet{wu2022slotformer} are concerned with.
In CLEVRER-S, the total number of objects increases by two or more in 9.5\% of the videos, while in CLEVRER-L, this occurs for 68.8\% of the videos.


We evaluate FG-ARI and FG-mIoU on the two frames individually, which is then averaged.
We also compute a temporal consistency (TC) metric as the fraction of objects that are correctly captured by the same slot (details in \autoref{app:clevrer details}).

\paragraph{Results.}
\autoref{tab:CLEVRER results} shows that SA-MESH outperforms the other models by a significant margin.
The only exception is on CLEVRER-S for temporal consistency---we can close this gap simply by reducing the MESH learning rate $\lambda$ (five run average: 90.0\% FG-ARI, 22.7\% FG-mIoU, 95.1\% TC).
The proposal by \citet{wu2022slotformer} of adding noise to slots to prevent them from collapsing helps SA on CLEVRER-L, but to a lesser extent than SA-MESH, which uses exclusive multiset-equivariance to prevent collapse.
This is evidenced by the much better FG-mIoU of SA-MESH.

\section{Related work}

\paragraph{Object-centric learning.}

Object-centric learning aims to model visual input data in terms of multiple ``objects'' rather than a global representation or grid of feature vectors.
This decomposition can be considered an \emph{abstraction} of the input; reasoning over a small number of objects is intuitively more efficient than over a feature map \citep{ke2022, huang2020srn}.
Scenarios with multiple independent objects are ubiquitous in natural data, so it is desirable to model them well.
\citet{clevrtex} classify object-centric learning methods into three categories:
pixel-space approaches which group related pixels together~\citep{greff2019iodine,pervez2022differentiable}, glimpse approaches which sequentially extract patches from the input~\citep{crawford2019spatially,lin2020SPACE,jiang2020generative}, and sprite approaches which learn a dictionary of object appearances~\citep{monnier2021unsupervised,smirnov2021marionette}.
These are part of the wider research area of factorizing knowledge into smaller, independent parts which can be modeled more easily \citep{goyal2019rim,goyal2021objectfiles,didolkar2021nps}.

We choose to apply MESH on specifically the pixel-space based slot attention \citep{locatello2020object} because of its simplicity in approach (cross-attention with GRU updates), the lack of assumptions on what an object is (which makes it a general technique), and its set-equivariance.
Since \citet{zhang2022multisetequivariant} show a specific limitation with set-equivariance, there is a clear path towards improvement, namely making it exclusively multiset-equivariant.
We accomplish this in this paper, with strong results backing up the benefits of our approach.

\paragraph{Optimal transport.}
Another approximation of optimal transport can be obtained through the Sliced Wasserstein Distance \cite{bonneel2015sliced}.
It performs tiebreaking through the use of numerical sorting and is thus exclusively multiset-equivariant, but it lacks precise 1-to-1 associations between inputs and slots.
This can especially be a problem with varying input sizes.
We tried approaches based on this method as a replacement for standard cross-attention but did not obtain any competitive results.

In a similar direction to Sinkhorn, which performs entropy-regularized optimal transport, \citet{blondel2018smooth} study L2-regularized optimal transport problems.
While their solver is faster than unregularized optimal transport and obtains lower entropy solutions than Sinkhorn, similarly to Sinkhorn the convexity of the problem means that it cannot break ties effectively on its own, even with noise.
In SA-MESH, we can replace Sinkhorn with this method, but we found that it was too slow comparatively.

\section{Discussion}

We introduced several variants of slot attention that can break ties between slots which enables better modeling of objects.
In particular, MESH is a promising method that enhances cross-attention. As a result, it grants slot attention the property of exclusive multiset-equivariance while maintaining learnability and efficiency.


While minimizing entropy in MESH has a nice symmetry with entropy-regularized optimal transport, it is not clear whether a sparse attention map is always desirable.
We try to learn a neural network on the transport map in \autoref{app:MESH objective}, but find no improvements over simply using the entropy; the derivative of the learned objective ends up with a similar shape to that of the entropy, which suggests that entropy is indeed a reasonable choice to minimize for now.
While the experiments on ClevrTex are a small step towards more complicated image data, we do not have any evaluation on real world data, so it is not certain what new problems will present themselves. Recent object-centric learning techniques that are able to scale to real-world scenarios often use more powerful image encoders and decoder architectures with the vanilla slot attention, so it is possible in principle to simply replace SA with SA-MESH in these models. In practice, it is so far uncertain whether any inductive biases introduced by SA-MESH only apply well on simpler synthetic data.

Our experiments show that in certain cases, SA-SH can already provide most of the benefits without the additional complexity of the bi-level optimization in SA-MESH. For example, most of the benefit over SA in \autoref{sec:copying} is already obtained with SA-SH, while SA-MESH only provides a small benefit over SA-SH. On the other hand, there are cases like \autoref{sec:od video} where SA-MESH greatly outperforms SA-SH. We believe that it is important to gain a better understanding of what situations make one preferable over the other.

A limitation of our experiments is that we only evaluate the MESH idea in the context of slot attention, when in reality it is a more general method. For example, it could be used to enhance self-attention in Transformers. 
Another example is that Sinkhorn is used by \citet{pena2022rebasin} for merging the weights of two neural networks together.
MESH could be used as an alternative in this context to replace the Sinkhorn algorithm. In general, we believe that optimal transport will continue to play an important role in deep learning, with MESH being a way of bringing tiebreaking into the picture without paying the usual speed penalty associated with optimal transport.


\section*{Acknowledgements}
The work of DWZ is part of the research programme Perspectief EDL with project number P16-25 project 3, which is financed by the Dutch Research Council (NWO) domain Applied and Engineering Sciences (TTW).
This research was enabled in part by compute resources provided by Mila (\url{mila.quebec}), Calcul Québec (\url{calculquebec.ca}), the Digital Research Alliance of Canada (\url{alliancecan.ca}), and by support from the Canada CIFAR AI Chair Program. Simon Lacoste-Julien is a CIFAR Associate Fellow in the Learning Machines \& Brains program.

\bibliography{bibliography}
\bibliographystyle{icml2023}

\appendix

\section{Marginals in Sinkhorn and EMD}\label{app:marginals}
As we mention in the main text, we need to account for the (typical) case of the number of inputs $n$ and the number of slots $m$ differing, i.e.\ with a cost matrix $\mC \in \R^{m \times n}$.
The problem is that it is impossible to make every row and every column of the transport map sum to 1 when the number of rows and columns is different.
Fortunately, there is standard practice for how to deal with this case in optimal transport \cite{peyre2019ot, cuturi2013sinkhorn,bonneel2011displacement}.
We can define non-negative marginals $\va \in \R^m$ and $\vb \in \R^n$ that specify the row and column sums of the transport map respectively.
If $\sum_i \va_i = \sum_j \vb_j$, then convergence is as normal.

In our case, we learn both $\va = m \cdot \softmax(h_\va(\mZ))$ and $\vb = m \cdot \softmax(h_\vb(\mX))$ with neural networks $h_\va: \R^d \to \R$ and $h_\vb: \R^c \to \R$ that are shared across the $m$ slots or $n$ input elements respectively.
These allow the model to put focus on important input elements (e.g. the inputs corresponding to objects) and ignore unimportant input elements (e.g. the inputs corresponding to the background), as well as put focus on the relevant number of slots.
Since both softmaxes sum to one, we have $\sum_i \va_i = \sum_j \vb_j = m$ so there is no problem with convergence.

For the Sinkhorn algorithm, it now repeatedly alternates normalizing all the rows to sum to $\va$, then all the columns to sum to $\vb$.
For the EMD solver that we use \cite{bonneel2011displacement}, these marginals are standard parameters in the algorithm.
In the main text, we omit these marginals whenever we refer to $\mathrm{sinkhorn}$ or $\mathrm{emd}$ for simplicity of notation.

\section{Enforcing $\mC'$ to be similar to $\mC$}\label{app:similarity}
The following discussion is not critical to understanding the main text, since we find empirically that with the right initialization (e.g. $\mC' = \mC + \bm\eps$ with $\bm\eps_{ij} \thicksim \mathcal{N}(0, 10^{-6})$), learning is not a problem.
We only find that this variant is necessary with an initialization such as $\eps \thicksim \mathcal{N}(0, \mI)$.
As we mention in the main text, the amount of noise is not important as long as it remains above machine precision after applying $\mathrm{sinkhorn}$, which can be easily checked a-priori.

In order to enforce $\mC'$ to be related to $\mC$ more explicitly, we define the following objective instead:
\begin{align}
    \mathrm{MESH}(\mC) = \argmin_{\mC'} [H(\sink(\mC')) \nonumber \\
        \quad\quad\quad + \alpha || \mathrm{sinkhorn}(\mC') \mS - \mathrm{sinkhorn}(\mC) ||^2] \label{eqn:me2}
\end{align}
The second term relates $\mC'$ to $\mC$ directly with a regularization factor $\alpha$.
$\mS$ is a similarity matrix, which we will define shortly.
The idea behind it is to allow costs to be freely changed among similar slots, but disallow this for dissimilar slots.
The aim of $|| \mathrm{sinkhorn}(\mC') \mS - \mathrm{sinkhorn}(\mC) ||^2$ is thus to make sure that $\mathrm{sinkhorn}(\mC')$ looks the same as $\mathrm{sinkhorn}(\mC)$ after allowing weight in the transport map to be moved around among similar slots.

\paragraph{Example}
Consider the case where we have three slots: $\mZ = [\vx, \vx, \vy]$ and three inputs $\mX = [\bm{\alpha}, \bm{\beta}, \bm{\gamma}]$.
Let us assume for this example that the cost matrix prefers associating $\bm{\gamma}$ with $\vy$ and both $\bm{\alpha}$ and $\bm{\beta}$ with $\vx$.
Computing unregularized optimal transport solutions would therefore give us either of two solutions:
\begin{equation}
    \mT_1 =
    \begin{bmatrix}
    1 & 0 & 0 \\
    0 & 1 & 0 \\
    0 & 0 & 1 \\
    \end{bmatrix}
    \quad
    \text{or}
    \quad
    \mT_2 =
    \begin{bmatrix}
    0 & 1 & 0 \\
    1 & 0 & 0 \\
    0 & 0 & 1 \\
    \end{bmatrix}
\end{equation}

However, the Sinkhorn algorithm is unable to break the tie between the two $\vx$ slots, so even with a temperature approaching 0, we obtain the following result:

\begin{equation}
    \mathrm{sinkhorn}(\mC) =
    \begin{bmatrix}
    0.5 & 0.5 & 0 \\
    0.5 & 0.5 & 0 \\
    0 & 0 & 1 \\
    \end{bmatrix}
\end{equation}

Suppose we have a similarity matrix $\tilde{\mS} \in \R^{m \times m}$ ($m$ is the number of slots) that measures pairwise similarities ranging from 0 to 1:
\begin{equation}
    \tilde{\mS} =
    \begin{bmatrix}
    1 & 1 & 0 \\
    1 & 1 & 0 \\
    0 & 0 & 1 \\
    \end{bmatrix}
\end{equation}
The first two slots are similar amongst themselves but dissimilar to the $\vy$ slot.
If we normalize each column of $\tilde{\mS}$ to sum to 1 to obtain $\mS$, then we see the following:
\begin{align}
    \underbrace{
        \begin{bmatrix}
        1 & 0 & 0 \\
        0 & 1 & 0 \\
        0 & 0 & 1 \\
        \end{bmatrix}
    }_{\mT_1}
    \underbrace{
        \begin{bmatrix}
        0.5 & 0.5 & 0 \\
        0.5 & 0.5 & 0 \\
        0 & 0 & 1 \\
        \end{bmatrix}
    }_{\mS}
    &=
    \underbrace{
        \begin{bmatrix}
        0.5 & 0.5 & 0 \\
        0.5 & 0.5 & 0 \\
        0 & 0 & 1 \\
        \end{bmatrix}
    }_{\mathrm{sinkhorn}(\mC)}
    \\
    \underbrace{
        \begin{bmatrix}
        0 & 1 & 0 \\
        1 & 0 & 0 \\
        0 & 0 & 1 \\
        \end{bmatrix}
    }_{\mT_2}
    \underbrace{
        \begin{bmatrix}
        0.5 & 0.5 & 0 \\
        0.5 & 0.5 & 0 \\
        0 & 0 & 1 \\
        \end{bmatrix}
    }_{\mS}
    &=
    \underbrace{
        \begin{bmatrix}
        0.5 & 0.5 & 0 \\
        0.5 & 0.5 & 0 \\
        0 & 0 & 1 \\
        \end{bmatrix}
    }_{\mathrm{sinkhorn}(\mC)}
\end{align}

This means that $\mathrm{sinkhorn}(\mC') = \mT_1$ and $\mathrm{sinkhorn}(\mC') = \mT_2$ are both valid solutions for the minimization of $|| \mathrm{sinkhorn}(\mC') \mS - \mathrm{sinkhorn}(\mC) ||^2$.
Note that any other permutation matrix for $\mT$ (i.e.\ one where there is not a 1 in the bottom right corner) would not be a valid solution.
This restricts \autoref{eqn:me} to only consider transport maps that are convex combinations of $\mT_1$ and $\mT_2$ for this example, with the entropy minimization preferring $\mT_1$ and $\mT_2$ specifically.
The small amount of noise in the $\mC'$ initialization arbitrarily makes it prefer one of the two.

\paragraph{Definition}
We define the similarity matrix $\tilde{S}_{ij} = g(\mZ_i, \mZ_j)$, where $g: \R^c \times \R^c \to \R$ is a small neural network that takes pairs of slots as input and produces a similarity score as output.
We then normalize each column of $\tilde{\mS}$ to sum to 1 by applying softmax on each column.
\begin{equation}
    \mS = \softmax(\tilde{\mS})
\end{equation}
If we set up a training task for the example described above, we observe that $\mS$ is learned to be virtually the same as the $\mS$ we use in the example.

\section{Multiset-equivariance of SA-SH, SA-EMD, and SA-MESH}\label{app:proof}
First, we show that SA-SH is set-equivariant, and therefore not exclusively multiset-equivariant.

\begin{proposition}
SA-SH is set-equivariant.
\end{proposition}
\begin{proof}
Slot attention is set-equivariant \citep{locatello2020object}.
All the additional operations in SA-SH are set-equivariant for a similar reason to Deep Sets \citep{zaheer2017deep}: only sum, broadcast, and elementwise operations are used in Sinkhorn.
Since composition of set-equivariant operations maintains set-equivariance, SA-SH is set-equivariant.
\end{proof}

\begin{proposition}
SA-EMD and SA-MESH are exclusively multiset-equivariant.
\end{proposition}
\begin{proof}
    To show exclusive multiset-equivariance, we need to show that they are not set-equivariant, but still multiset-equivariant. We begin with the former.
    
To show that SA-EMD and SA-MESH are not set-equivariant, it is enough to give a counter-example. Suppose we have the following cost matrix: 

\begin{equation}\label{eq:symmetric input}
    \begin{bmatrix}
        1 & 1 \\
        1 & 1 \\
    \end{bmatrix}
\end{equation}

Running EMD gives us one of the following transport maps as the solution, depending on the arbitrary tiebreaking in the EMD implementation.

\begin{equation}
    \begin{bmatrix}
        1 & 0 \\
        0 & 1 \\
    \end{bmatrix}
    \quad \text{or} \quad
    \begin{bmatrix}
        0 & 1 \\
        1 & 0 \\
    \end{bmatrix}
\end{equation}

Both have zero entropy, hence they are also possible solutions when the MESH objective is perfectly optimized.
Set-equivariance requires that a permutation applied to the input in \autoref{eq:symmetric input} changes the output by the same permutation. As pointed out by \citet{zhang2022multisetequivariant}, this does not happen because the arbitrary tiebreaking remains the same. 
Therefore, both SA-EMD and SA-MESH are not set-equivariant.

To show that they are multiset-equivariant, first recall the definition of multiset-equivariance.
\begin{equation}
\begin{split}
     \forall \mX \in \R^{n \times c}, \forall \mP_1 \in \Pi, {\exists \mP_2 \in \Pi}: \\
     f(\mP_1 \mX) = {\mP_2} f(\mX) {~\land~\mP_1 \mX = \mP_2 \mX}.
\end{split}
\end{equation}

EMD produces a solution with the minimum total cost by definition. This means that the transport map must remain the same, up to permutation.
This is because if any of the values in the transport map were to change (aside from being permuted), then the original solution was not a minimum, which is a contradiction. Therefore, we know that a $\mP_2$ must exist for any $\mP_1$.


In the same way with MESH, with infinitesimally small noise, the only thing that can change is the permutation of the solution: if no ties are broken the noise has virtually no effect because the subsequent operations in MESH are continuous, if a tie is broken then the ordering of the tie is random.
In either case, we can again always find a $\mP_2$ for every $\mP_1$ on the inputs, since the values in the solution remain the same up to permutation.

We have thus shown that SA-EMD and SA-MESH are not set-equivariant, but are multiset-equivariant (i.e. exclusively multiset-equivariant).
\end{proof}

\section{Sinkhorn algorithm implementation}\label{app:sinkhorn}
Ideally, we want to compute the Sinkhorn algorithm for as few steps as possible since it is used in every MESH step, which means that we also have to differentiate through the Sinkhorn algorithm in every MESH step.
However, we also need to run it for a sufficient number of steps for (good enough) convergence.
Fortunately, because we repeatedly run the Sinkhorn algorithm on \emph{similar} inputs over different MESH steps, we can optimize its implementation.
The idea is that the gradient descent for minimizing the entropy makes small changes (especially near MESH convergence), which allows us to reuse computation between different MESH steps.

\begin{figure}[htpb!]
\centering
\includegraphics[width=\linewidth]{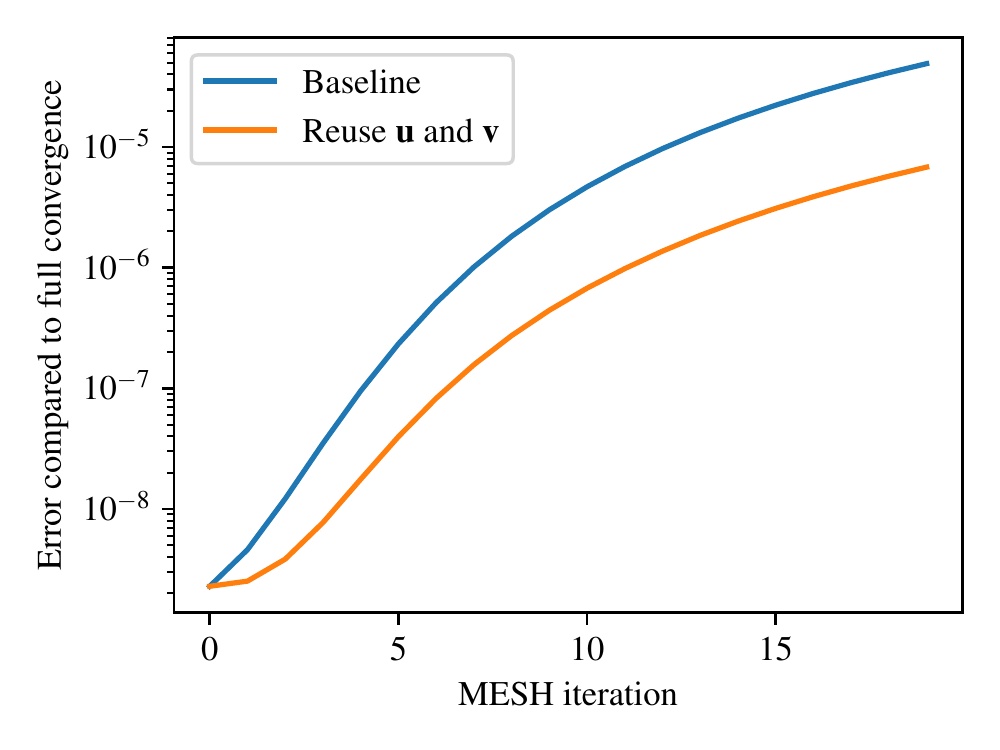}
\caption{Mean absolute error gap to the fully converged Sinkhorn for the two different Sinkhorn implementations for varying numbers of MESH iterations. Both implementations always use 5 Sinkhorn iterations and thus have comparable computational costs. Reusing $\vu$ and $\vb$ is clearly more effective.}
\label{fig:approximation}
\end{figure}

If the entropy minimization has converged, then the Sinkhorn output does not change.
This means that if we keep track of the operations we performed in the previous SA-MESH iteration, we can simply reapply them.
Let us take a look at these operations:
the Sinkhorn algorithm (multiplicatively) rescales rows and columns repeatedly.
Multiplication is commutative, so we can collect all the row normalizations together into a single row normalizer $\vu$ and all the column normalizations into a single column normalizer $\vv$.
Applying these two normalizers on the (exponentiated) cost matrix gives us exactly the same result as if we had manipulated the matrix with the normalizations directly.
The next time we run the Sinkhorn algorithm with a slightly changed cost matrix, we can bootstrap the algorithm with the previously-found $\vu$ and $\vv$.
We can then run a few more iterations to account for the changes in the cost matrix, giving us a new $\vu$ and $\vv$.

In \autoref{fig:approximation}, we ablate whether there are benefits to reusing $\vu$ and $\vv$ in SA-MESH. In particular, we compute the ideal solution by running the Sinkhorn algorithm until convergence at every ME iteration. Then we examine whether reusing $\vu$ and $\vv$ helps close the gap to the ideal solution when we limit the number of SH iterations.
At 1 ME iteration there are no $\vu$ and $\vv$ for bootstrapping available, so both exhibit the same gap.
At more than 1 ME iteration we observe that reusing $\vu$ and $\vv$ helps in narrowing the gap to the ideal solution, or equivalently, can achieve the same approximation with fewer iterations.


We do not claim that this technique of collecting Sinkhorn operations into $\vu$ and $\vv$ is novel, as there are several implementations that use this trick for performing the Sinkhorn algorithm.
Usually, it is a minor implementation detail since the approaches of manipulating the matrix directly and collecting normalizations into $\vu$ and $\vv$ are mathematically equivalent.
In our case however, this formulation leads to a concrete benefit due to our setup where we run the Sinkhorn algorithm on similar inputs, which allows us to reuse $\vu$ and $\vv$ in a beneficial way.

\section{Gradients of tempered SH and MESH, without normalizing their scale}\label{app:tempsh vs mesh}
In \autoref{fig:grad2}, we show the same gradient norms as in \autoref{fig:grad}.
However, rather than normalizing each model to have a maximum of 1 (which is more useful for visualizing them all at once), we maintain their native scaling.
This shows that changing the SH temperature has a major effect on the scale of the gradients while changing the MESH learning rate only has minor effects on the gradients (but still significant effects on entropy reduction).
In other words, the amount of entropy minimization can be changed without major impacts on the hyperparameters of other parts of the network.

\begin{figure*}
    \centering
    \includegraphics[width=0.49\linewidth]{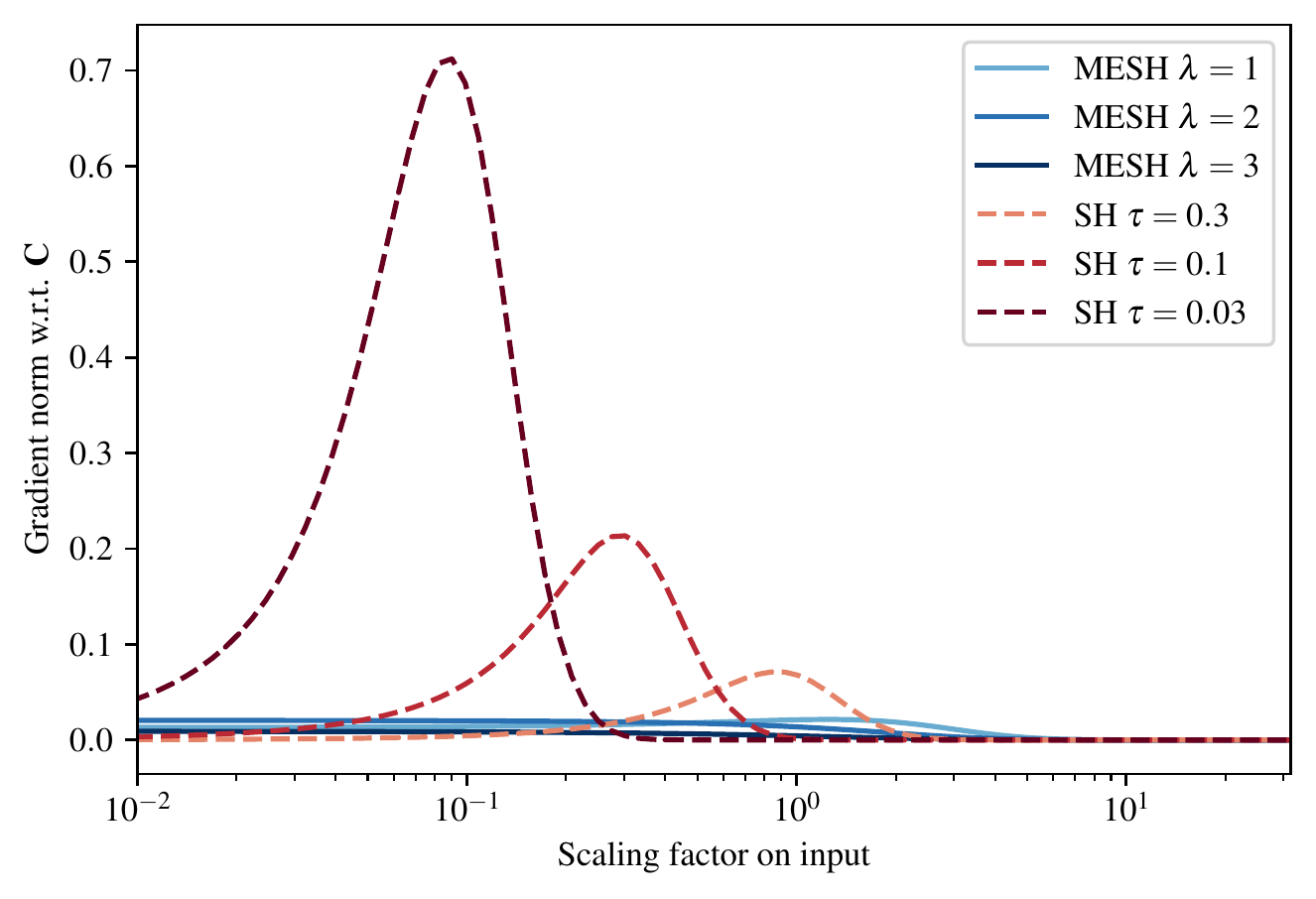}
    \includegraphics[width=0.49\linewidth]{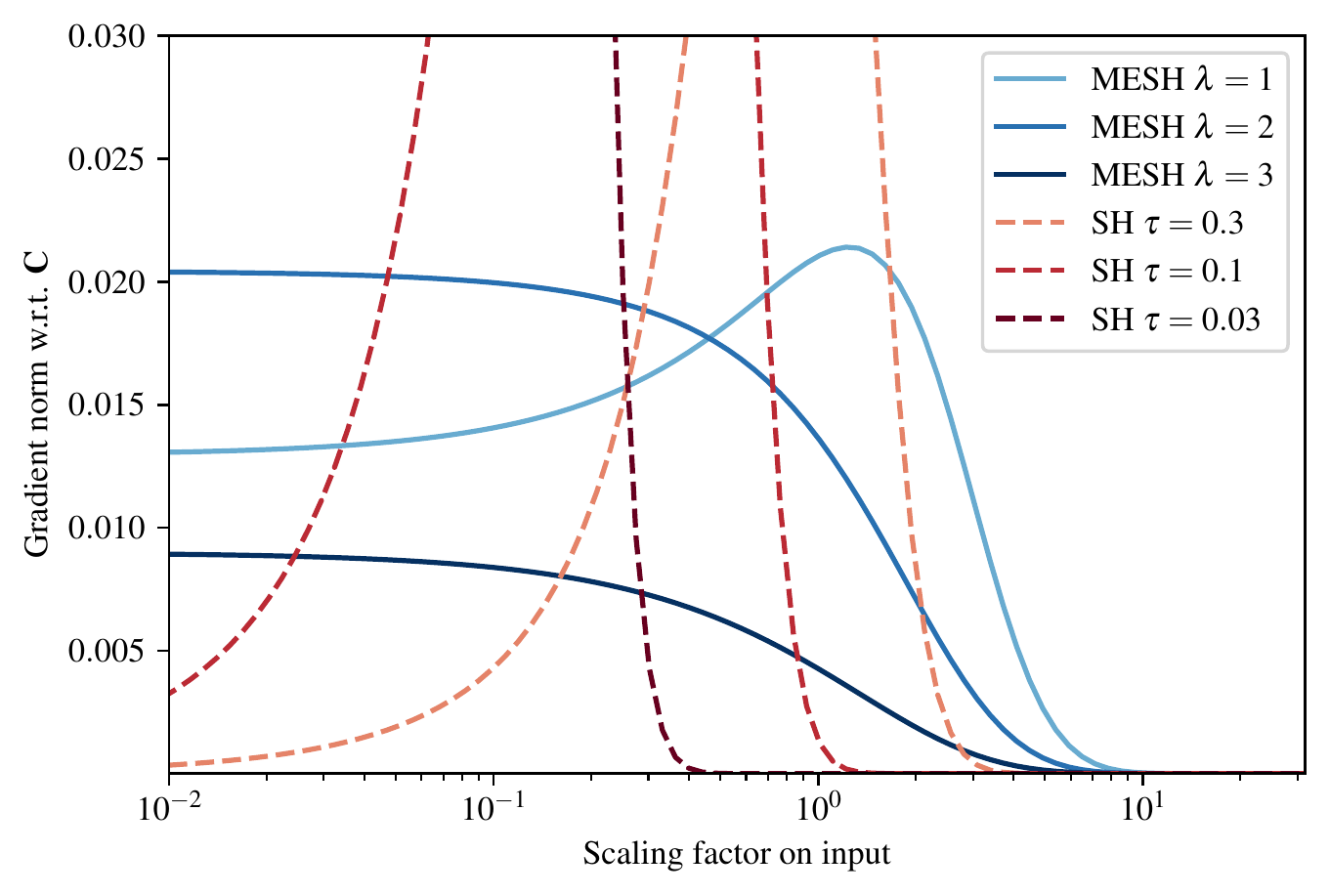}
    \caption{Same gradient norm plot as \autoref{fig:grad} but without normalizing each model to have a maximum of 1. The only difference between left and right is the scale of the y-axis. This shows that the gradient norm for tempered SH varies drastically with temperature, while the gradient norm of MESH remains more similar across learning rates. This makes tuning the learning rate hyperparameter easier as it can be considered more independently from other hyperparameters of the model.}
    \label{fig:grad2}
\end{figure*}

\section{Experimental details}\label{app:experiment details}
\subsection{Random object detection}
We generate a dataset of 64,000 data points to train on, each being a multiset with five 32-dimensional objects sampled from $\mathcal{N}(0, \sigma^2 \mI)$ and 100 zero vectors.
We directly apply slot attention on this: the dimensionality of the slot attention weights are all 32.
Since we know that there are always five objects, we set the number of slots to five.
The loss is computed by computing a mean squared error between all pairs of predicted and ground-truth objects, then using the Hungarian algorithm find the matching with the lowest loss. 
We find that using implicit differentiation of slots \citep{chang2022object} is significantly more stable on this dataset, so we use it for all models.
We train all models for 20 epochs with a batch size of 64 (1,000 steps each epoch).

\subsection{Object detection on CLEVR}
We largely follow same training setup as DSPN and iDSPN \citep{zhang2019dspn, zhang2022multisetequivariant} and adapt the slot attention implementation to it.
Matching \citet{zhang2022multisetequivariant} and \citet{locatello2020object}, we resize the input images to 128x128.
To compute the loss, we use the Hungarian algorithm to compute the least-cost matching between predicted objects and ground-truth objects.

\subsection{Object discovery} \label{app:object discovery details}
We compute the temporal consistency (TC) by first matching the predicted objects in each frame to their corresponding ground-truth objects. Then, we calculate the proportion of objects that have the same slot in both frames out of those that appear in both frames.
To find the optimal assignment between the predicted and ground-truth objects, we first compute all pairwise IoUs between the predicted and ground-truth masks. We then invert these IoUs by applying $1 - \text{IoU}$ and use the Hungarian matching algorithm to find the best match---the one with the highest mIoU.

We compute the mIoU in a similar manner as the TC metric, by finding the matching between the ground-truth objects and the predicted objects that results in the highest mIoU.

\paragraph{Multi-dSprites.}
We closely follow the experimental setup described by \citet{locatello2020object}. Specifically, we use the same image encoder, decoder, and hyperparameters where applicable. \citet{locatello2020object} used 500k training steps, while all of our runs were trained for 530 epochs, which results in slightly fewer than 500k training steps.

\paragraph{ClevrTex.}
We pre-process the images by applying the same center crop as suggested by~\citet{clevrtex} and resize the images to 64$\times$64 resolution instead of 128$\times$128 resolution. This allows us to use the same neural network architecture as we did for the Multi-dSprites dataset.

Since the dataset has more complicated visuals we increase the model size by increasing the channel sizes. In particular, we double the number of channels in the image encoder and decoder to 64, and we double the dimensions of the slots to 128 (with the MLP in slot attention having an intermediate dimension of 256). We again train all models for 530 epochs which correspond to around 330k gradient update steps in this case. The maximum number of objects in an image is 10, so we set the number of slots to 11.

\subsection{CLEVRER} \label{app:clevrer details}
We construct the datasets from the 20k videos in the CLEVRER dataset, by picking two frames at specific timesteps from each video. In the CLEVRER-S dataset, we use the first and 16th video frames. In the second dataset, we use the first and 128th (last) frames. The proportion of examples where new objects appear increases with the time gap between the two frames, and similarly the amount of displacement for objects that are in both frames increases too.
We show examples from the two dataset variants in \autoref{fig:clevrer examples}.

For the distance function $d$ that computes the cost matrix of the optimal transport problem in SA-MESH, we empirically find that the cosine distance works better in this case than the $l2$ distance. We suspect that since the $l2$ distance allows the slots to be pushed arbitrarily far apart that learning might slow down in the later stages of training. 

We extend the model which we used in the Multi-dSprites experiment to video data. Our setup is similar to \citet{kipf2022conditional}, but we do not use a predictor model (except for the SA learned noise baseline) to update the slots when transitioning from one video frame to the next.
In particular, the model first applies the image encoder to all video frames independently to compute the input feature maps. Next, the SA (or our proposed variants) is applied to the features of one video frame at a time, and every time the slots are initialized from the slots of the previous frame. Finally, each image is decoded independently. The learned noise baseline uses a 2-layer MLP with LayerNorm to predict the mean and variance of a Gaussian, from which the initial slots are sampled for the next frame, following the stochastic SAVi setup by \citet{wu2022slotformer}.
We use 8 slots.

\section{CLEVR object prediction results}\label{app:clevr extra}
In \autoref{tab:state appendix} we show our results for CLEVR object prediction without implicit differentiation of slots \citep{chang2022object}.
All results are slightly lower than the results reported in \autoref{tab:state}, but the overall message remains exactly the same. The only major difference is that SA-SH performs much worse compared to SA-SH$\dagger$.

\begin{table*}[t]
    \centering
    \caption{
        Results on CLEVR object property multiset prediction,
        average precision (AP) in \% (mean $\pm$ standard deviation) over 5 random seeds, higher is better.
        All SA results are based on our re-implementation.
        SA (original) results copied from \citet{locatello2020object}, iDSPN results from \citet{zhang2022multisetequivariant}.
    }
    \label{tab:state appendix}
\resizebox{0.90\textwidth}{!}{
    \begin{tabular}{l *{6}{d{2.4}}c}
        \toprule
        \mc{Model} & \mc{AP\textsubscript{$\infty$}} & \mc{AP\textsubscript{1}} & \mc{AP\textsubscript{0.5}} & \mc{AP\textsubscript{0.25}} & \mc{AP\textsubscript{0.125}} & \mc{AP\textsubscript{0.0625}} & Time \\

        \midrule
        \textit{iDSPN} \cite{zhang2022multisetequivariant} & \spmi{98}{8}{0.5} & \spmi{98}{5}{0.6} & \spmi{98}{2}{0.6} & \spmi{95}{8}{0.7} & \spmi{76}{9}{2.5} & \spmi{32}{3}{3.9} & \mc{---}\\
        SA (original) \cite{locatello2020object} & 94.3\spm{1.1} & 86.7\spm{1.4} & 56.0\spm{3.6} & 10.8\spm{1.7} & 0.9\spm{0.2} & \mc{---} & \mc{---}\\
        \midrule
        SA & 89.1\spm{1.2} & 85.7\spm{1.0} & 73.3\spm{1.2} & 35.4\spm{1.5} & 9.0\spm{0.8} & 2.0\spm{0.3} & 2.4 h \\
        SA-SH & 95.6\spm{1.0} & 94.0\spm{1.1} & 84.5\spm{1.7} & 41.3\spm{3.0} & 10.4\spm{0.7} & 2.5\spm{0.4} & 2.5 h\\
        SA-EMD & 99.2\spm{0.2} & 98.7\spm{0.4} & 97.0\spm{0.8} & 82.4\spm{1.2} & 34.0\spm{2.2} & 8.3\spm{0.9} & 9.7 h \\
        \textbf{SA-MESH} & \boldc{99.2\spm{0.3}} & \boldc{99.1\spm{0.3}} & \boldc{98.8\spm{0.5}} & \boldc{88.3\spm{0.8}} & \boldc{40.8\spm{1.0}} & \boldc{10.6\spm{0.3}} & 2.5 h \\
        \bottomrule
    \end{tabular}%
}

\end{table*}

\section{Extra ablations}\label{app:extra experiments}
\paragraph{Additional slot attention iterations}
In general the benefit of more iterations is minor (see ablations in \citet{locatello2020object}, Appendix C) and using too many can hurt in some cases, which is why many recent works \citep{wu2022slotformer,kipf2022conditional} set the number of iterations to 3 or even fewer. We ran an additional experiment where we trained the plain slot attention baseline on Multi-dSprites with 5 iterations resulting in 81.9±6.1 FG-ARI, which is worse than the 92.2±0.5 achieved with 3 iterations reported in our main results. Also note that the slot attention module is only a part of the full neural network and is not the bottleneck when using larger encoders.

\paragraph{Learned slot initializations}
In our perspective, the initialization should be thought of as separate to the slot attention method itself. A different initialization does not change the fact that the slots can collapse, especially in cases like the video datasets where the initialization is not a free parameter but dependent on the previous timestep. Thus, having control over the initialization should not be relied upon.

\citet{locatello2020object} report in their Appendix B that learning the initial slots decreases the performance in unsupervised learning. We ran experiments with SA using a learned initialization on Multi-dSprites to evaluate this as well. On the Multi-dSprites dataset, SA with a learned initialization achieves 93.0±1.0 FG-ARI, which is comparable to the standard SA at 92.2±0.5 and remains lower than the 95.6±0.2 of SA-MESH.

\section{Alternative MESH objective} \label{app:MESH objective}
In~\autoref{sec:MESH} we choose the entropy as the inner objective function because the goal was to reverse the effect of the entropy-regularized optimal transport version. Alternatively, it is possible to learn a neural network with scalar inputs and outputs in place of the entropy function. For the neural network, we choose a simple 2-layer MLP with ReLU activations and 32 hidden dimensions. We plot the derivative of a learned objective function in~\autoref{fig:learned mesh objective function}. Its shape is similar to the derivative of the entropy~\autoref{fig:entropy mesh objective function}, but learning it incurs additional compute compared to simply using the entropy function $H$.

Justified by this analysis we can directly use the entropy function as the MESH objective for improved efficiency. Empirically we observe that we do not even need to backpropagate through the gradient descent optimization procedure of MESH and it suffices to treat MESH as the identity function during backprop. One perspective that might explain this: since the negative derivative of $H$ is a monotonic increasing function, changes to the input will also affect the output in the same direction.

\begingroup
    \captionsetup{type=figure}
    \centering
    \begin{subfigure}{0.23\textwidth}
        \centering
        \includegraphics[width=0.98\linewidth]{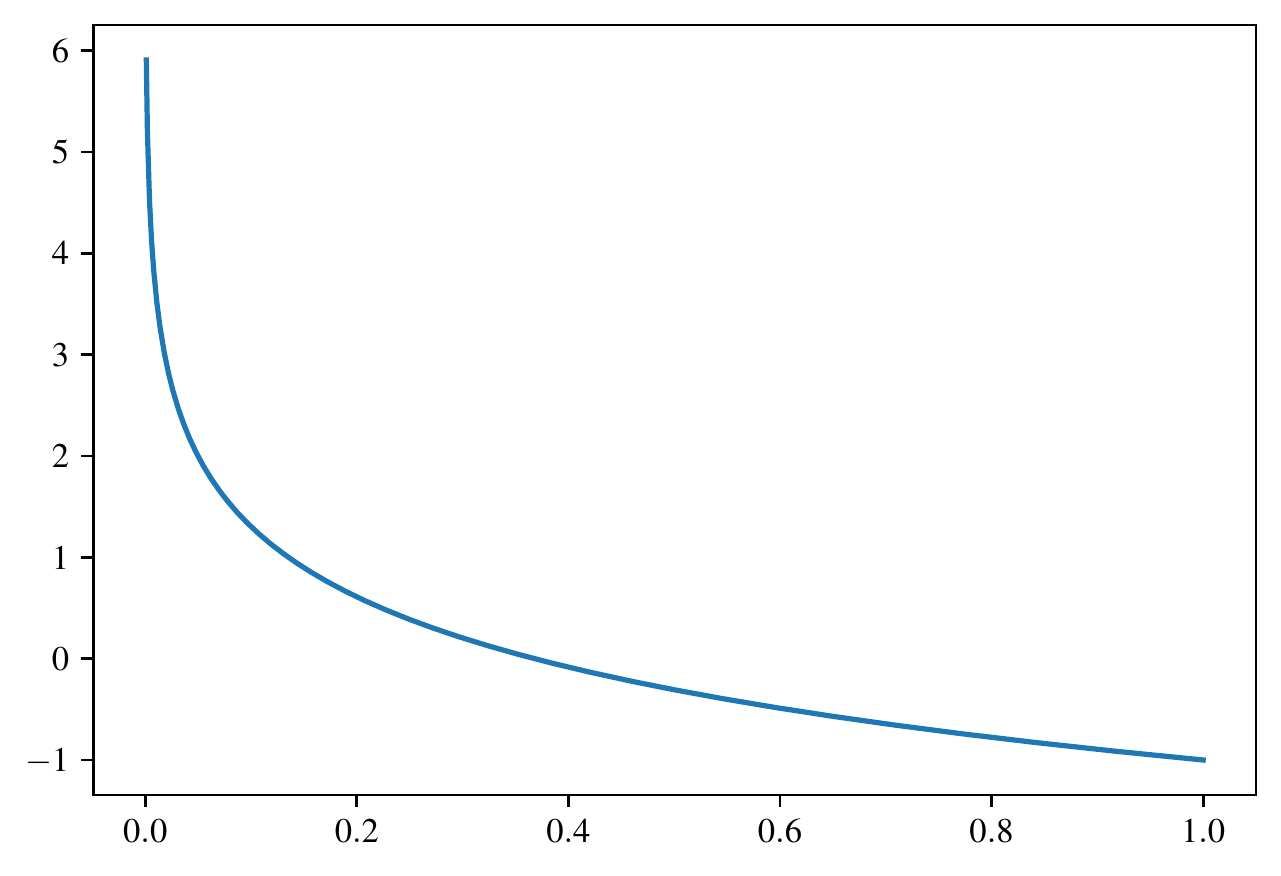} 
        \caption{$H$} \label{fig:entropy mesh objective function}
    \end{subfigure}%
    \begin{subfigure}{0.23\textwidth}
        \centering
        \includegraphics[width=0.98\linewidth]{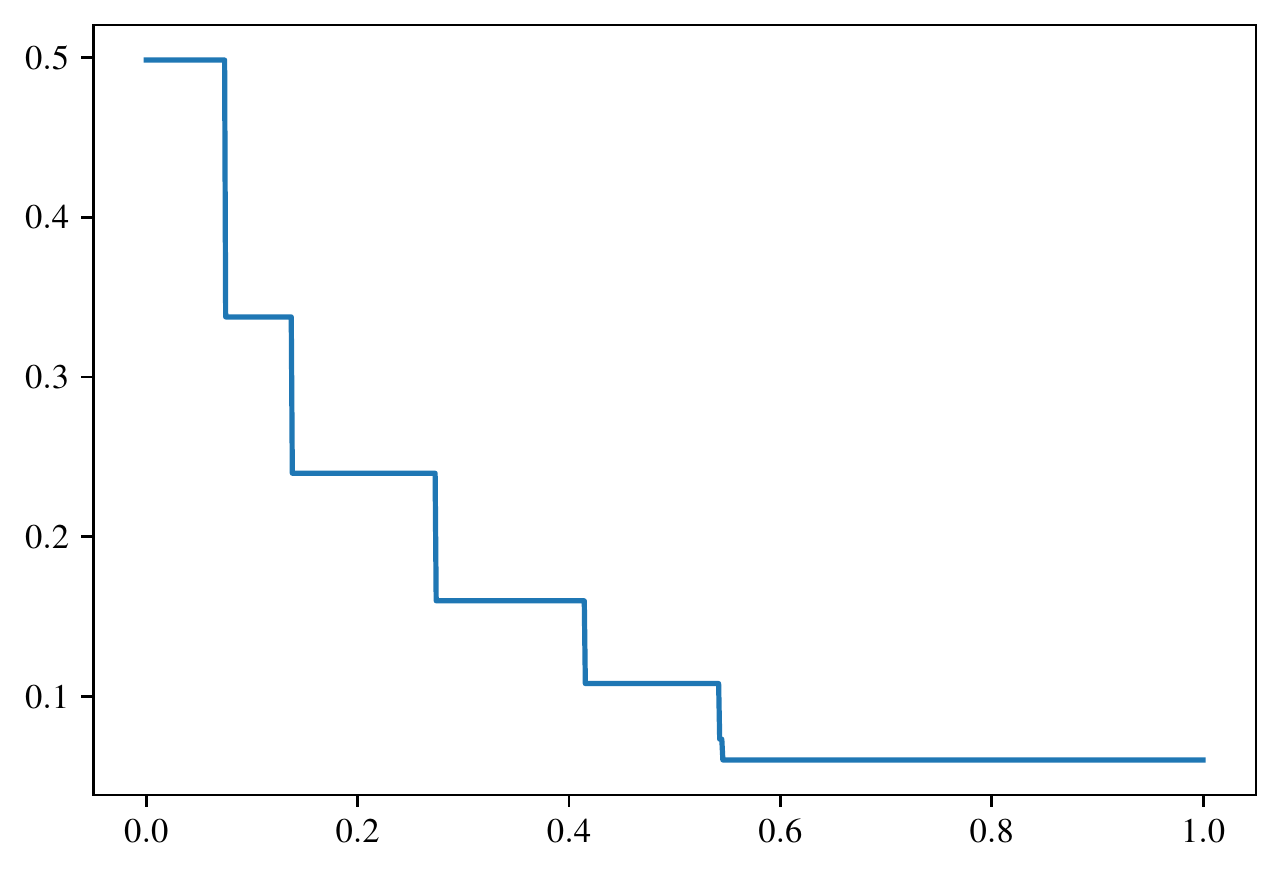}
        \caption{MLP} \label{fig:learned mesh objective function}
    \end{subfigure}
    \caption{Derivative of the objective function in MESH}
    \label{fig:mesh objective function}
\endgroup

\begin{figure*}[t]
    \centering
    \begin{subfigure}{0.4\linewidth}
        \centering    
        \includegraphics[width=0.77\linewidth]{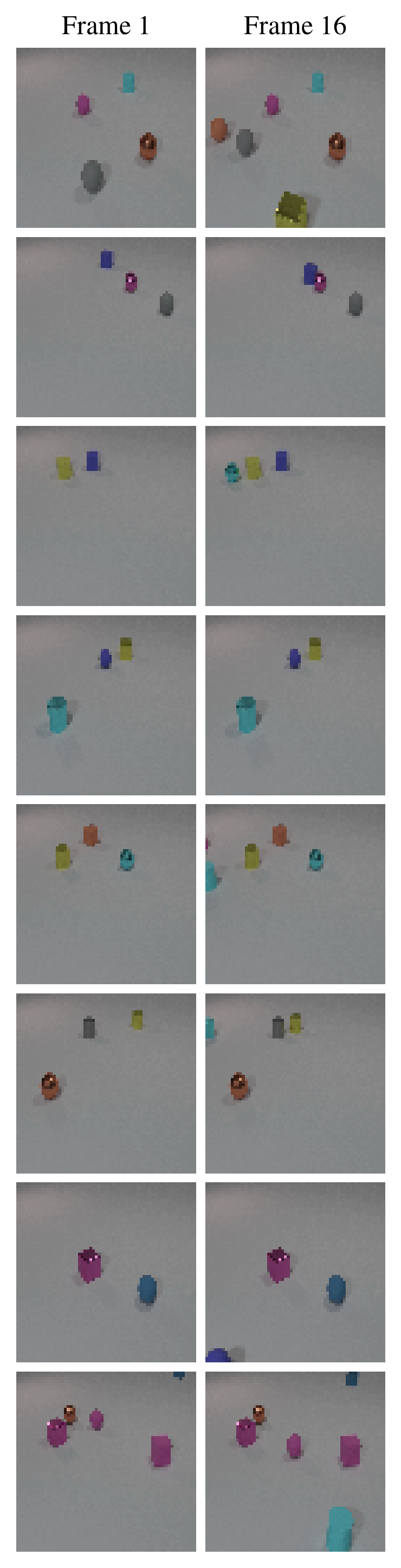}
        \caption{CLEVRER-S}
    \end{subfigure}%
    \begin{subfigure}{0.4\linewidth}
        \centering    
        \includegraphics[width=0.77\linewidth]{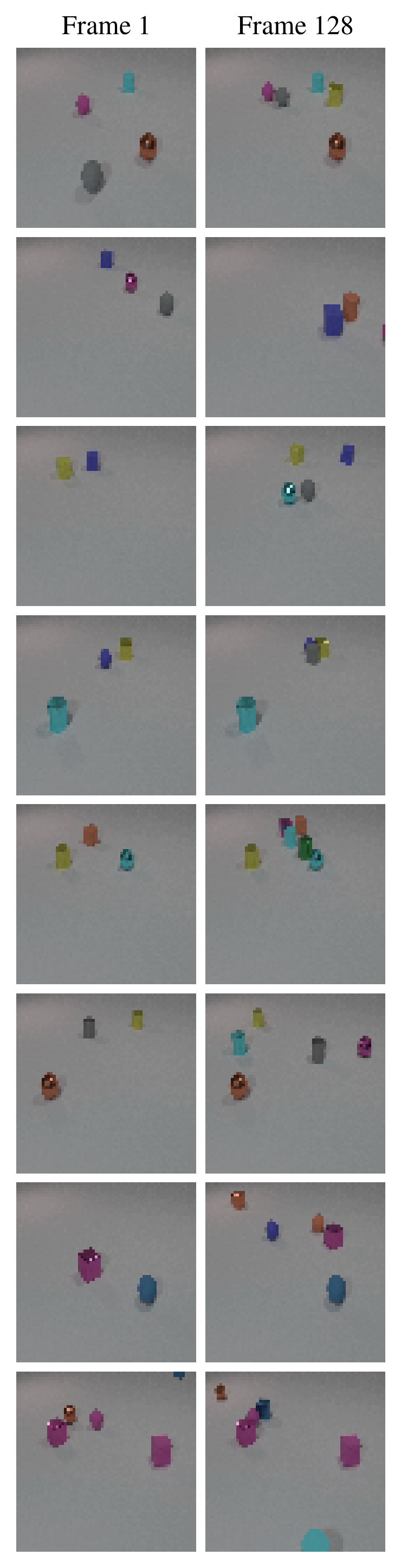}
        \caption{CLEVRER-L}
    \end{subfigure}
    \caption{Examples from the two datasets derived from CLEVRER. Significant changes like multiple new objects appearing occur less frequently in CLEVRER-S. Objects can be significantly displaced from one frame to the other in CLEVRER-L.}
    \label{fig:clevrer examples}
\end{figure*}

\section{Object discovery example results}\label{app:od examples}
In the following, we show examples of SA and SA-MESH performing object discovery on the various datasets that we use.
We always show the original image on the left, followed by either the attention maps for each slot or the final alpha masks for each slot.
The attention map or the alpha masks are multiplied with the original image to make it easier to tell how precise their locations are.

\begin{itemize}
    \item \autoref{fig:attn map all} shows the intermediate attention maps over the three slot attention iterations on the Multi-dSprites dataset.
    \item \autoref{fig:attn map all2} shows the final per-slot alpha masks on the Multi-dSprites dataset.
    \item \autoref{fig:attn map and masks clevrtex} shows the intermediate attention map of the last slot attention iteration, as well as the final per-slot alpha masks on the ClevrTex dataset.
    \item \autoref{fig:clevrer mask examples} shows the final per-slot alpha masks on the CLEVRER-S and CLEVRER-L datasets.
\end{itemize}

Please refer to the individual figure captions for a more detailed description of observations on these results.

\begin{figure*}
    \centering
    \begin{subfigure}{0.49\linewidth}
        \centering
        \includegraphics[width=\linewidth]{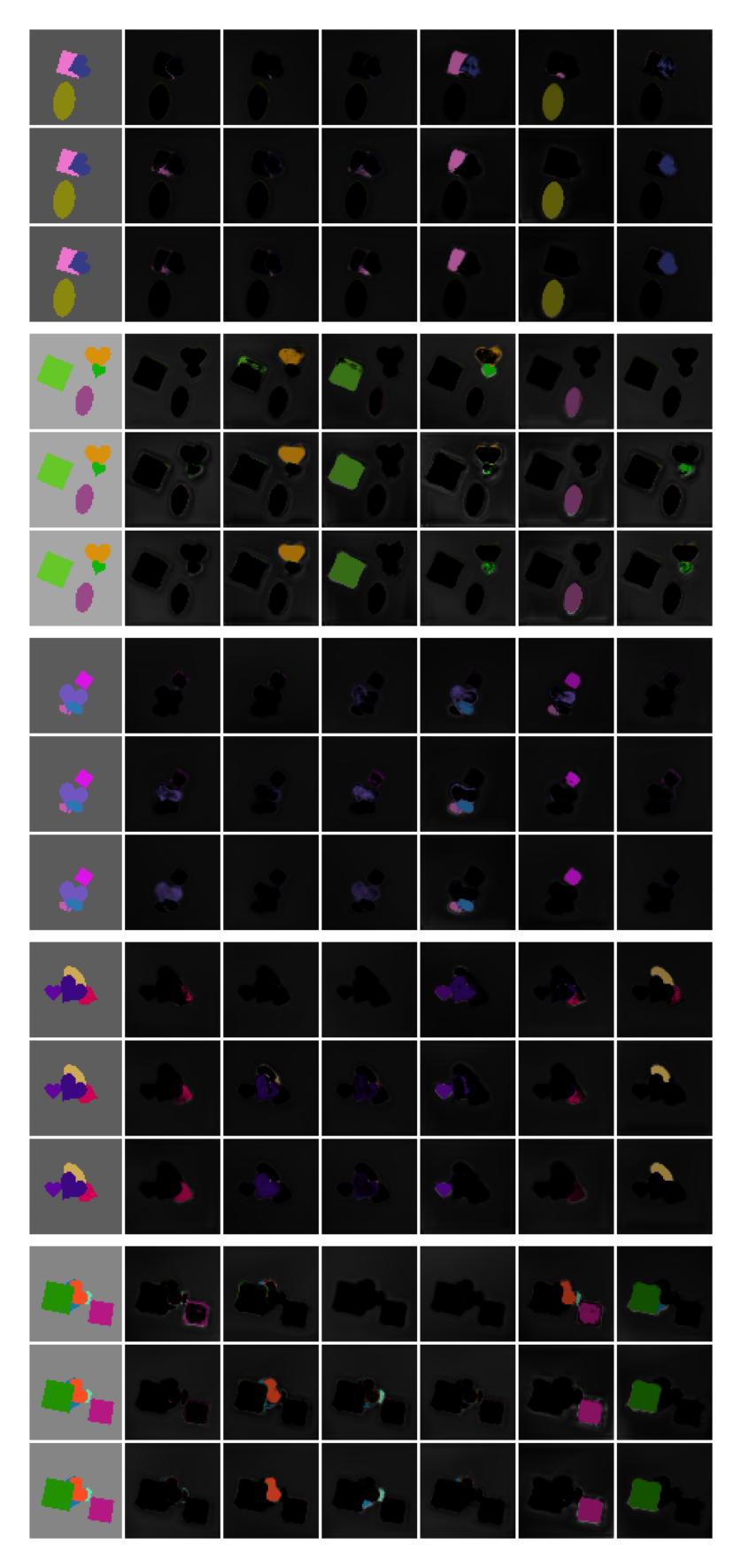}
        \caption{SA}
    \end{subfigure}%
    \begin{subfigure}{0.49\linewidth}
        \centering
        \includegraphics[width=\linewidth]{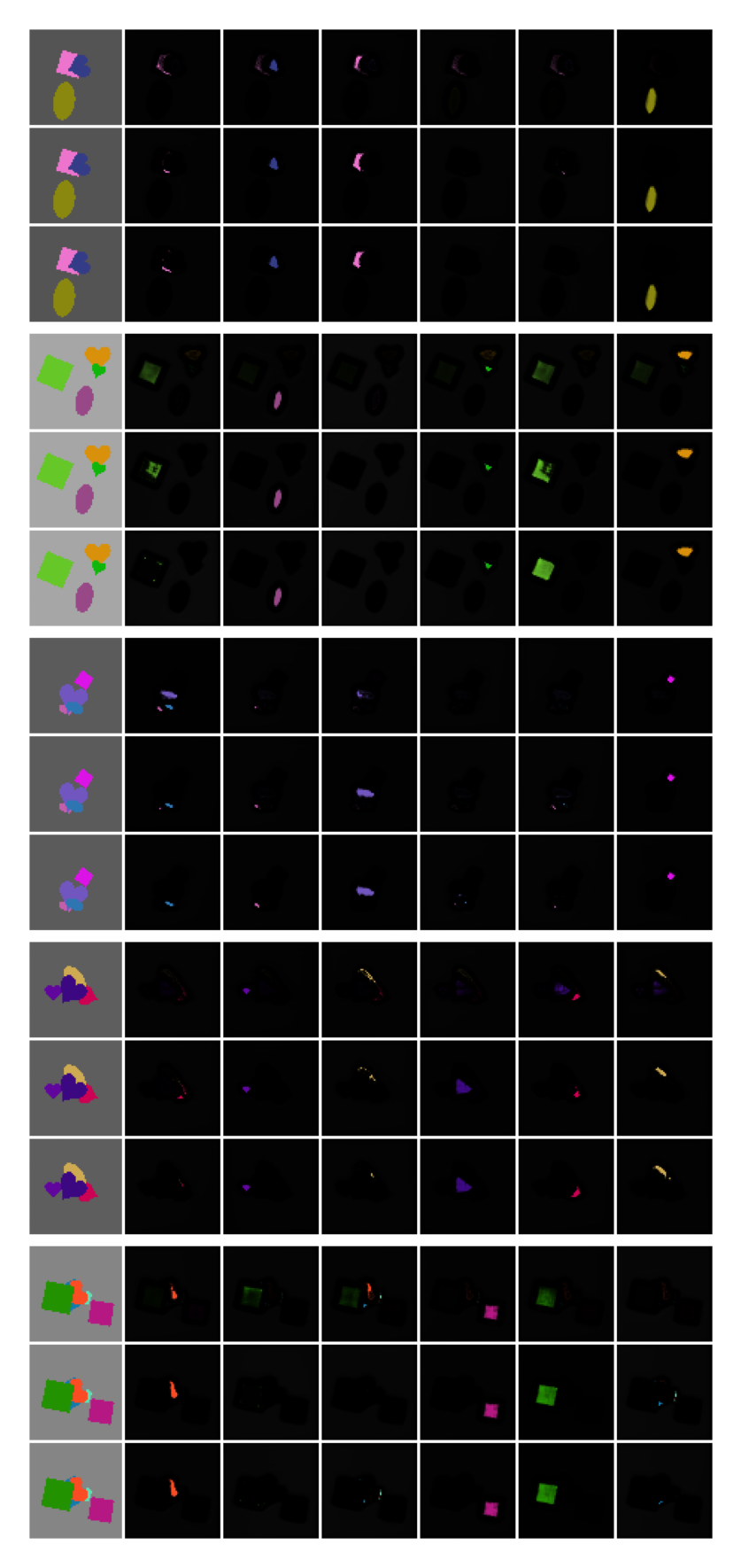}
        \caption{SA-MESH}
    \end{subfigure}%
    \caption{Attention maps for all three slot attention iterations for five different examples from the validation split of Multi-dSprites. Note how the shade of SA is generally darker indicating higher attention values even in background areas. In the third example, we can see how SA models the pink and blue ellipses using one slot while splitting the purple heart over two slots. In contrast, SA-MESH is able to route the three objects into three different slots in the second slot attention iteration.}
    \label{fig:attn map all}
\end{figure*}

\begin{figure*}
    \centering
    \begin{subfigure}{0.49\linewidth}
        \centering
        \includegraphics[width=\linewidth]{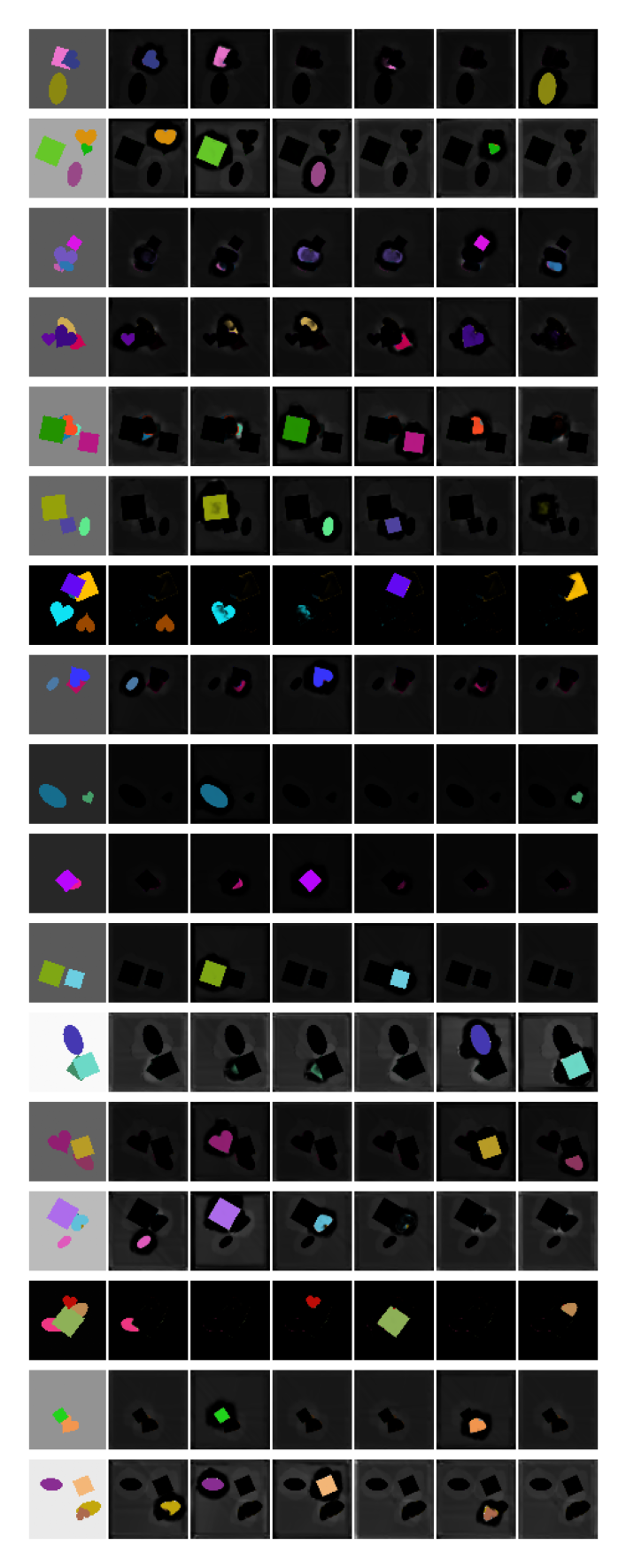}
        \caption{SA}
        \vspace{-4mm}
    \end{subfigure}%
    \begin{subfigure}{0.49\linewidth}
        \centering
        \includegraphics[width=\linewidth]{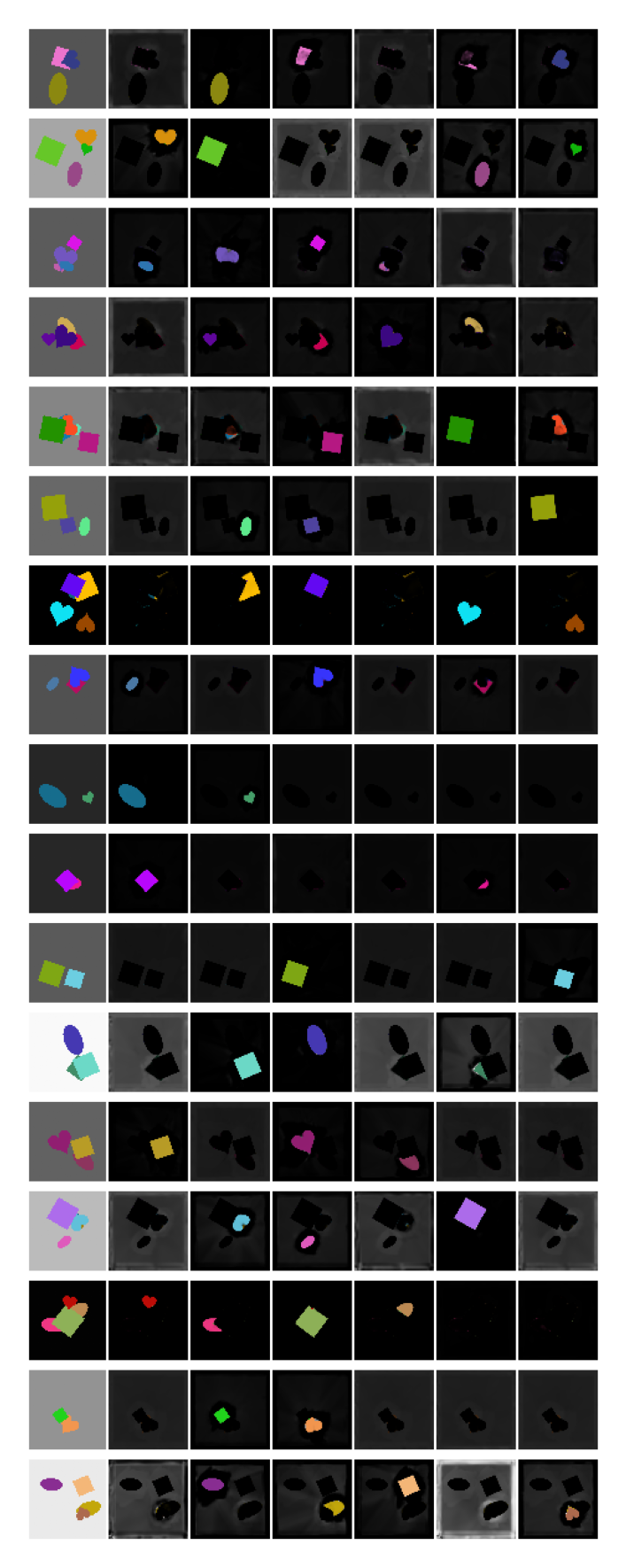}
        \caption{SA-MESH}
        \vspace{-4mm}
    \end{subfigure}%
    \caption{Predicted alpha masks from the validation split of Multi-dSprites. In general, the masks for SA-MESH are sharper than for SA. For example, in the last row, the brown heart is only recognizable as a blown blob for SA, but is a distinct heart shape for SA-MESH. In the third example, SA splits the purple heart into two slots, while SA-MESH models it with one as desired.}
    \label{fig:attn map all2}
\end{figure*}

\begin{figure*}
    \centering
    \begin{subfigure}{\linewidth}
        \centering
        \includegraphics[width=0.8\linewidth]{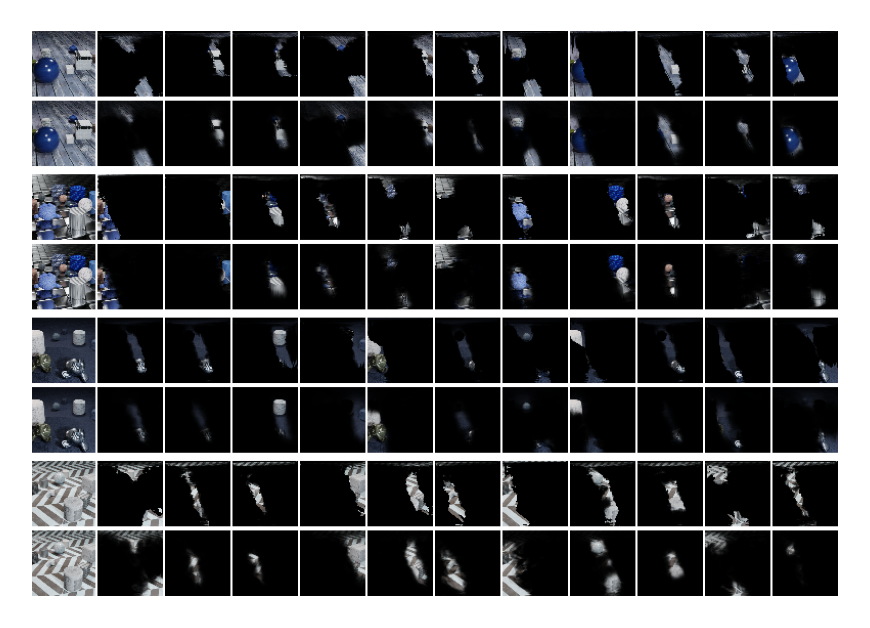}%
        \vspace{-2mm}
        \caption{SA}
    \end{subfigure}
    \begin{subfigure}{\linewidth}
        \centering
        \includegraphics[width=0.8\linewidth]{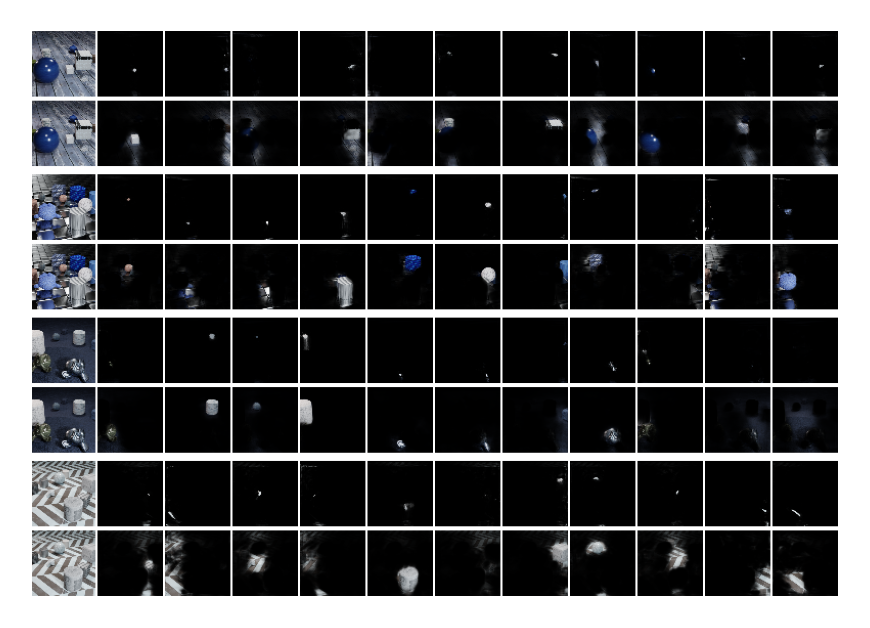}
        \vspace{-2mm}
        \caption{SA-MESH}
    \end{subfigure}
    \caption{Attention maps (top) and masks (bottom) for four different examples from the validation split of ClevrTex. SA commonly learns that each slot should attend to a spatial region as opposed to a specific object. SA-MESH on the other hand is better able to localize individual objects, though the background is often still handled with a region-specific approach. }
    \label{fig:attn map and masks clevrtex}
\end{figure*}


\begin{figure*}[t]
\centering
\resizebox{\linewidth}{!}{%
\begin{tabular}{l@{}c@{}c}
&CLEVRER-S&CLEVRER-L\\  
\rotatebox[origin=c]{90}{SA}&\raisebox{-.48\height}{\includegraphics[width=0.5\linewidth]{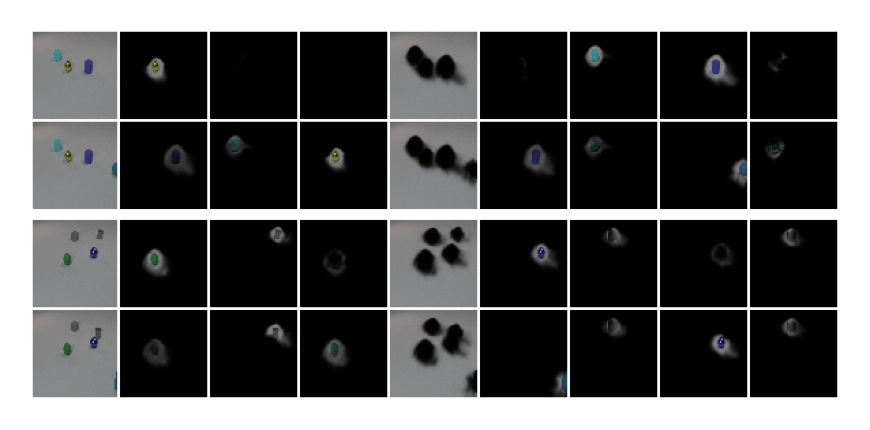}} & 
\raisebox{-.48\height}{\includegraphics[width=0.5\linewidth]{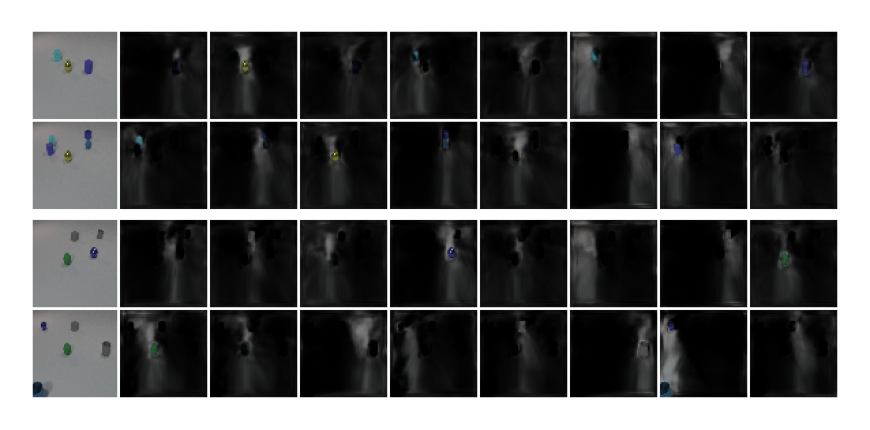}} \\
\rotatebox[origin=c]{90}{SA fixed noise}&\raisebox{-.48\height}{\includegraphics[width=0.5\linewidth]{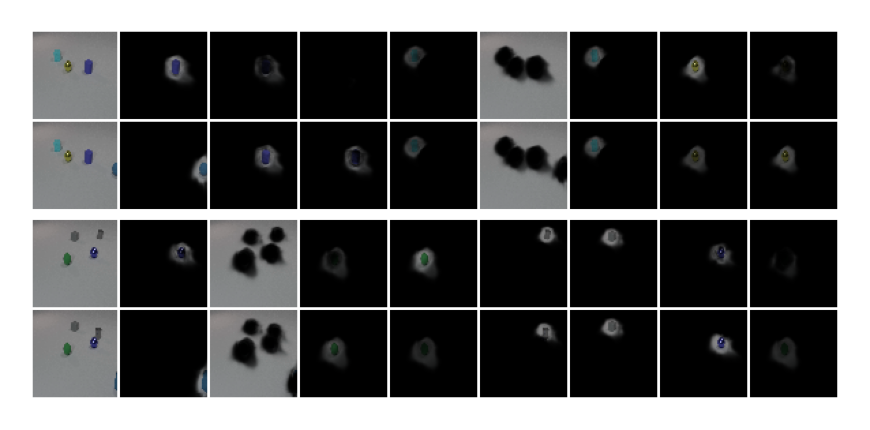}} & 
\raisebox{-.48\height}{\includegraphics[width=0.5\linewidth]{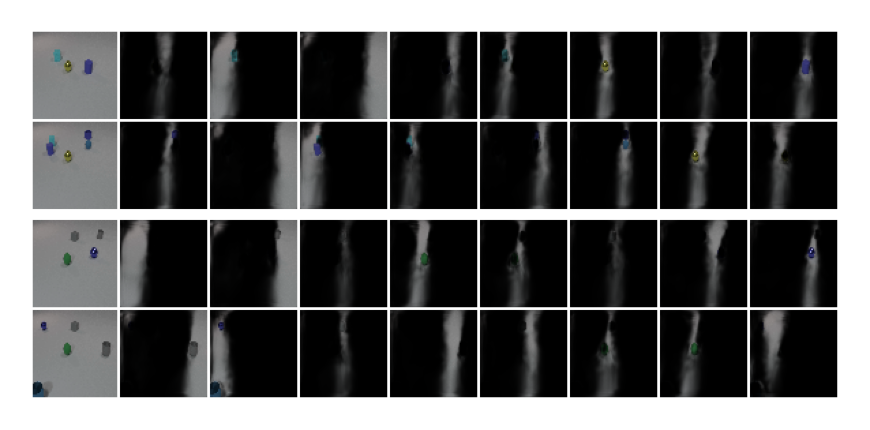}} \\
\rotatebox[origin=c]{90}{SA learned noise}&\raisebox{-.48\height}{\includegraphics[width=0.5\linewidth]{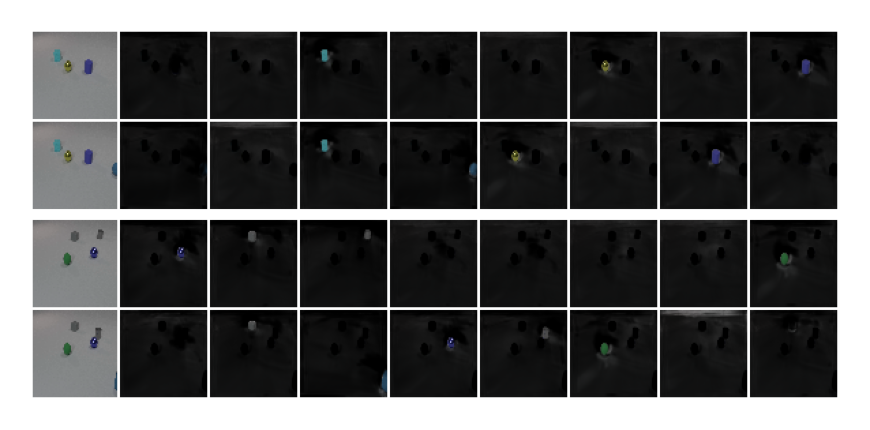}} & 
\raisebox{-.48\height}{\includegraphics[width=0.5\linewidth]{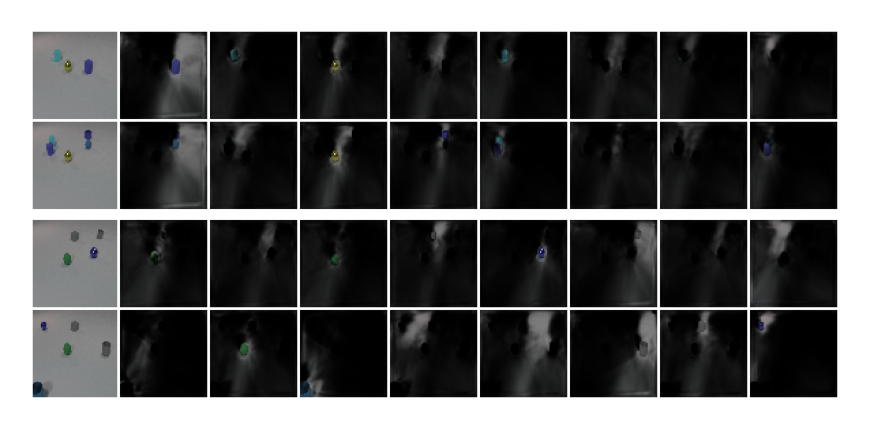}} \\
\rotatebox[origin=c]{90}{SH}&\raisebox{-.48\height}{\includegraphics[width=0.5\linewidth]{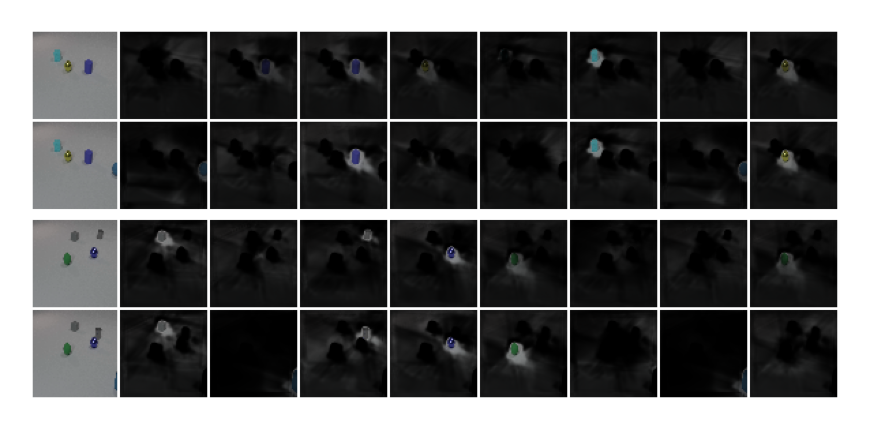}} & 
\raisebox{-.48\height}{\includegraphics[width=0.5\linewidth]{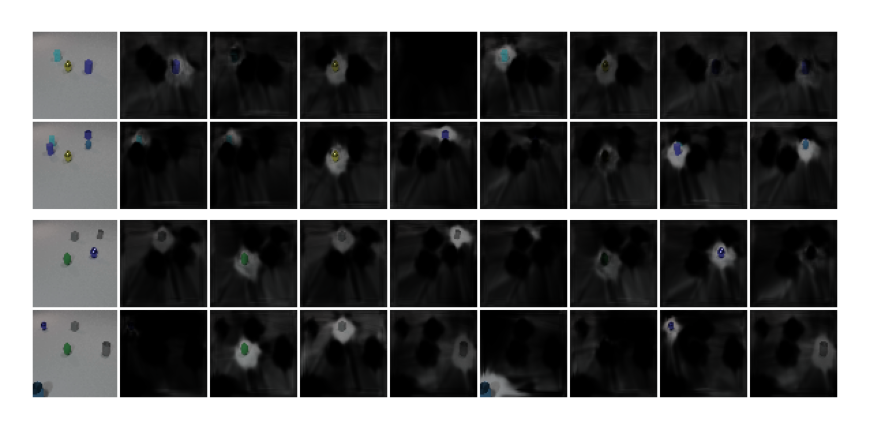}} \\
\rotatebox[origin=c]{90}{SA-MESH}&\raisebox{-.48\height}{\includegraphics[width=0.5\linewidth]{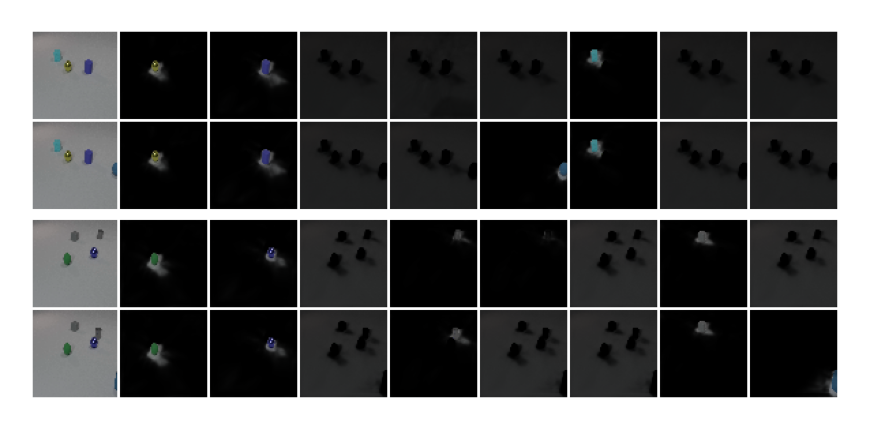}} & 
\raisebox{-.48\height}{\includegraphics[width=0.5\linewidth]{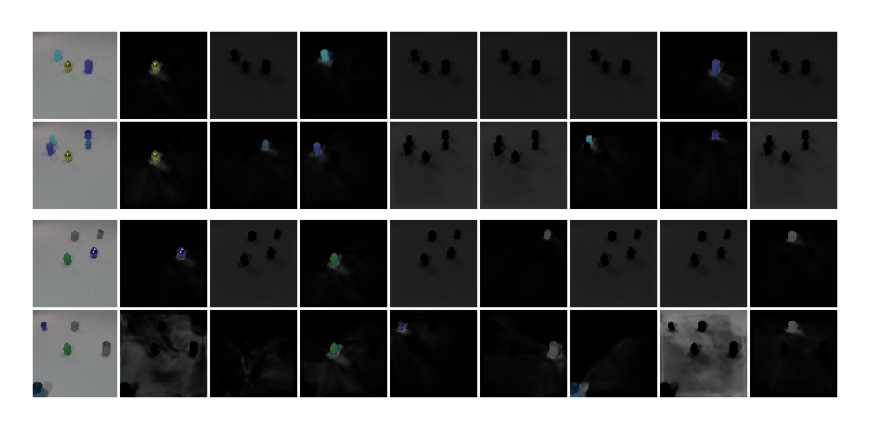}}
\end{tabular}%
}
\caption{Example alpha masks for CLEVRER-S and CLEVRER-L. Possibly due to the difficulty of handling multiple new objects, SA and its noise variants choose a region-based decomposition on CLEVRER-L instead of an object-based decomposition like on CLEVRER-S. SA-SH and SA-MESH learn an object-based decomposition for both datasets.} \label{fig:clevrer mask examples}
\end{figure*}

\end{document}